\documentclass[hidelinks]{article}
\def\shownotes{1}

\usepackage[main,final,nonatbib]{neurips_2026}

\usepackage[utf8]{inputenc} %
\usepackage[T1]{fontenc}    %
\usepackage{hyperref}       %
\usepackage{url}            %
\usepackage{booktabs}       %
\usepackage{amsmath,amsthm,amsfonts,amssymb}       %
\usepackage{nicefrac}       %
\usepackage{microtype}      %
\usepackage[dvipsnames]{xcolor}         %
\usepackage{algorithm}
\usepackage{algpseudocode}
\usepackage{subcaption}
\usepackage{macro}
\usepackage{multirow}

\usepackage[backend=biber,
            natbib=true,
            maxbibnames=99,
            minalphanames=3,
            uniquename=false,
            uniquelist=false,            
            maxcitenames=2,
            sorting=ynt,
            giveninits=true,                        
            labeldateparts=true,
            sortcites=true,
            backref=false]{biblatex}
\title{Multi-Objective Reinforcement Learning Using Routed Ensembles and Trajectory Attribution}

\author{%
  Zhenshuo Zhang\\
  Northeastern University\\
  Boston, MA\\
  \And
  Ruoxuan Xiong\\
  Emory University\\
  Atlanta, GA\\
  \And
  Hongyang R. Zhang\\
  Northeastern University\\
  Boston, MA\\
}

\begin{document}

\maketitle

\begin{abstract}
    Multi-objective reinforcement learning is a fundamental problem for modern AI, with applications spanning post-training, reasoning models, and robotic control. A central challenge is to learn policies that balance competing objectives, such as helpfulness versus honesty in safety and alignment, or speed versus energy in robotics. The optimal policy can vary sharply with objective-weighting combinations, and the extent of this variation depends crucially on the degree of conflict among objectives, making per-preference policy learning sample-inefficient. In this paper, we introduce objective-conflict coefficients and derive near-optimal regret bounds for online settings using this measure, thereby quantifying the impact of misalignment on learning efficiency. Then, we design routed ensemble policy learning, an algorithm that maintains a small set of basis policies and dynamically routes trajectories to the policy with the highest potential reward. Finally, we consider trajectory selection as another approach to mitigate objective conflict, and design a kernel surrogate trajectory attribution algorithm that estimates how feedback signals from a set of trajectories influence policy learning on other trajectories. The kernel can capture nonlinear interactions such as synergy and antagonism between trajectories. We evaluate our approach across diverse environments and benchmarks: routed ensembles improve average performance by 4.3\% over baselines for online environments from Meta-World and MO-Gymnasium, and kernel surrogate attribution improves correlation with leave-one-out analysis over existing baselines by 50\%. When applied to trajectory selection, kernel-based attribution yields a 27\% improvement in reward over baseline attribution methods.
\end{abstract}

\section{Introduction}\label{sec_intro}
Given $m$ potentially competing objectives in a reinforcement learning (RL) environment, how can we learn a small set of policies that maximizes expected reward under any specified preference weighting?
This problem is fundamental in multi-agent systems \cite{littman1994markov}, and is ubiquitous in modern AI systems. In robotics, agents often need to balance speed and energy efficiency. In reinforcement learning with human feedback, reasoning models, and post-training, models are often evaluated under competing criteria such as helpfulness, safety, and honesty \cite{joglekar2025training,sengupta2026multi}.

A core challenge in multi-objective RL (MORL) is that the optimal policy can vary drastically with objective weight combinations. For example, even small changes in preference can lead to qualitatively different optimal policies \cite{papadimitriou2000approximability,zhang2025scalable}. Prior work has focused on approximating the Pareto frontier of the optimal reward trade-offs \cite{van2014multi,yang2019generalized,zhang2025scalable}, or learning one policy for each objective \cite{wu2021accommodating,yang2024federated}, and it remains unclear \textit{whether or when} transferring trajectories across objectives is effective.

A key yet underexplored question is that the difficulty of learning across preferences is fundamentally tied to the degree of misalignment among objectives. When objective weightings are well aligned (e.g., rewarding both final and intermediate progress), additional objectives can facilitate learning, and data collected under one preference can benefit policy learning for others. When objectives are misaligned or adversarially-injected (e.g., RLHF with corrupted reward components), naive transfer can become inefficient or even harmful \cite{joglekar2025training,li2026moplex}. This problem is especially acute in high-dimensional settings, making it a critical open challenge to identify objective alignment \cite{joshibenchmarking}.

To address this challenge, we introduce a new theoretical framework based on measuring \emph{objective conflict coefficients} that quantify the degree of misalignment among objectives through the $\chi^2$-divergence. Our analysis shows that this coefficient interacts with the \textit{coverability coefficient}~\citep{xie2022role}---which captures how well a single exploration policy can cover the preference space---to determine the sample complexity of online MORL. Our regret bound generalizes single-objective results~\citep{agarwal2019reinforcement,jin2021bellman, xie2022role, chi2025statistical, chen2025outcome} and is the first to formalize how objective misalignment affects learning efficiency across preferences.
Our theoretical framework builds on recent advances in transfer learning under distribution shifts \cite{yang2025precise}. However, the technical challenge in MORL lies in quantifying the discrepancy in policy learning.
Our main result, Theorem \ref{theorem:regret}, relies on a low-rank basis assumption, along with a coverability condition that ensures sufficient exploration of the basis.

Motivated by our theoretical framework, we propose \textit{routed ensemble policy learning}, an online algorithm for learning a small set of basis policies, which can be applied to a real-world RL environment. This algorithm alternates between two steps: (i) routing each trajectory to the basis policy with the highest expected return under the corresponding preference weight, and (ii) partitioning trajectories into disjoint clusters and updating each basis policy using its assigned trajectories with a standard policy learning algorithm such as proximal policy optimization.

Next, we consider an alternative perspective for tackling objective conflicts, based on selecting trajectories. This motivates the problem of quantifying the influence of individual trajectories upon the learned optimal policy, a problem for which we term \emph{trajectory attribution}.
We design a scalable \emph{kernel-based trajectory attribution} procedure that estimates how trajectories help or hinder policy learning. Our procedure fits a kernel ridge regression surrogate to approximate policy performance across trajectory subsets. This can capture both individual contributions and higher-order interactions between trajectories. These interactions are common in high-dimensional settings and often underlie unexpected transfer behavior \cite{joshibenchmarking}. However, they are not captured by first-order influence approximations such as influence functions~\citep{rubinstein2025accuracy}, linear surrogate models \cite{liidentification,zhang2026efficient}, or gradient-similarity-based analysis \citep{wu2020understanding,hu2025snapshot}. Our approach goes beyond prior work by leveraging efficient approximations grounded in policy gradients and random projections \cite{zhang2025scalable}. The results from more accurate trajectory attribution can guide trajectory selection in downstream policy learning.

Lastly, we empirically validate our approach on the Meta-World robotics benchmark and high-dimensional MO-Gymnasium tasks. Our algorithm consistently outperforms multi-objective baselines, achieving a significant reduction in cumulative regret in environments with structured objective conflicts.
Specifically, routed ensemble learning outperforms the strongest baseline by $4.3\%$ on the Meta-World benchmark and achieves an $8\%$ average improvement in hypervolume on MO-Gymnasium tasks. 
Similar results are also observed in reinforcement learning with human feedback (RLHF) settings.
Additionally, kernel-based attribution provides a $50\%$ improvement in terms of the Spearman correlations with ground-truth leave-one-out influence scores.
When applied to downstream task selection, kernel-based selection yields a $27\%$ higher reward than existing attribution baselines.
We present detailed ablation studies to validate the theoretical insights from our proposed framework.
For instance, after the routed ensemble procedure, the coverability metric drops from $1.3$ for a single policy to $0.2$ for the ensemble policies, indicating that routing mitigates the exploration bottleneck caused by objective conflicts.
We also conduct a detailed analysis on the number of basis policies in the ensemble, as well as validate a low-rank basis assumption in our theoretical setup.

In summary, the contributions of this paper are threefold.
\begin{itemize}
    \item First, we introduce a new theoretical framework for MORL based on a $\chi^2$ measure that we call the objective conflict coefficient. The theoretical result can be relaxed to a setting with an approximate low-rank basis, and we empirically evaluate this assumption in various real-world environments.
    \item Second, we design two algorithms for tackling objective conflicts: (1) routed ensemble policy learning, and (2) kernel-based trajectory attribution to estimate inter-trajectory influence.
    \item Third, we demonstrate significant gains in sample efficiency and robustness across diverse MORL benchmarks covering language modeling, robotic control, and RLHF settings.
\end{itemize}

For the rest of this paper, we begin by stating the formal problem setup in Section \ref{sec_setup}.
Afterward, we present our theoretical framework and proposed approach to tackle objective conflicts in Section \ref{sec:approach}.
We will present our experiments to validate our proposed algorithms in Section \ref{sec:experiment}.
Lastly, we discuss the related work in Section \ref{sec_related} and conclude in Section \ref{sec_conclude}.
We provide the omitted algorithmic description and proofs in Appendix \ref{sec:algorithm} and \ref{sec_proof}.
Omitted experimental setup and empirical evaluations are provided in Appendix \ref{app:experiments}.
\section{Preliminaries}\label{sec_setup}

We consider an RL environment modeled as a \emph{multi-objective Markov decision process} (MOMDP), defined by the tuple
$\mathcal{M} = ( \mathcal{S}, \mathcal{A}, \Pr(\cdot \mid \cdot, \cdot), \mathbf{r}, \Omega, f_\Omega, H ),$
where $\mathcal{S}$ and $\mathcal{A}$ denote the state and action spaces, $\Pr(\cdot \mid s, a)$ is the transition kernel, $\mathbf{r} : \mathcal{S} \times \mathcal{A} \to \mathbb{R}^m$ is a vector-valued reward function with $m \ge 2$ objectives, and $H$ is the horizon. The preference space is denoted as $\Omega \subseteq \mathbb{R}^m_+$. The scalarization operator $f_w$ maps the return vector to a scalar utility under preference $w \in \Omega$. We evaluate performance under a target preference distribution $w \sim \cP$.

A stochastic policy $\pi \in \Pi$ maps states to action distributions, and induces a trajectory $\tau = (s_0, a_0, \mathbf{r}_0, \ldots, s_H, a_H, \mathbf{r}_H)$, where $a_h \sim \pi(\cdot \mid s_h)$ and $s_{h+1} \sim \Pr[\cdot \mid s_h, a_t]$. {More generally, the policy can be time dependent $\pi_1, \cdots, \pi_H$ with $a_h \sim \pi_h (\cdot\mid s_h)$ and induces a trajectory similarly.} 

The trajectory return is the cumulative reward vector
$\mathbf{G}(\tau) = \sum_{h=0}^{H}\mathbf{r}(s_h, a_h) \in \mathbb{R}^m \,.$
The expected return of policy $\pi$ from an initial distribution $s_0 \sim \rho$ is given by
\[ \mathbf{J}(\pi) = \exarg{s_0 \sim \rho, \tau \sim \pi}{\mathbf{G}(\tau)} \in \mathbb{R}^m. \]
We primarily focus on linear operators, and thus the policy's expected return can be written as $\mathcal{J}(\pi; w) = \langle w, \mathbf{J}(\pi) \rangle .$
We consider an episodic online MORL setting with episodes $t = 1, \ldots, T$. At the beginning of episode $t$, the system designer has access to the history $\mathcal{D}^{(t-1)} = \{(\tau^{(u)}, w_u)\}_{u=1}^{t-1}$ and observes the current preference vector $w_t\in\Omega$. Based on $(\mathcal{D}^{(t-1)},w_t)$, a learner applies a possibly randomized policy learning algorithm to select a policy $\pi^{(t)}$, interacts with the MOMDP for $H$ horizons, and observes a trajectory $\tau^{(t)} = (s^{(t)}_{0}, a^{(t)}_{0}, \mathbf{r}^{(t)}_{0}, \ldots, s^{(t)}_{H}, a^{(t)}_{H}, \mathbf{r}^{(t)}_{H})$. 

After $T$ episodes, the learner collects a set of trajectories $\mathcal{D} = \{(\tau^{(t)}, w_t)\}_{t=1}^T$ and aims to learn a preference-conditioned policy $\pi(\cdot;w)$ that attains high utility for target preferences $w$ in the support of $\cP$. In this paper, we focus on learning an ensemble policy corresponding to a small set of $\pi(\cdot; w)$.

Our theoretical framework builds on the notion of \emph{coverability}, which is a complexity measure defined on a fixed policy class $\Pi$ \cite{agarwal2019reinforcement,xie2022role}:
\[ C_\cov := \inf_{\mu \in \Delta(\mathcal{S} \times \mathcal{A})} \sup_{\pi \in \Pi} \max_{h \in [H]} \sup_{(s,a)} \frac{d_h^\pi(s,a)}{\mu_h(s,a)}, \]
where $d_h^\pi(s,a) = \Pr^\pi(s_h = s, a_h = a)$ denotes the state-action probability under policy $\pi$, and the infimum is over sequences of $\mu_h\in\Delta(\mathcal{S}\times\mathcal{A})$. Intuitively, $C_{\cov}$ captures how hard it is for a single exploration distribution to simultaneously cover the occupancy measures of all policies in $\Pi$, serving as a hardness measure for online exploration. %

\section{Our Approach}\label{sec:approach}

We first provide several examples to illustrate how multi-objective misalignment affects policy learning. We then formalize this observation through the notion of objective conflict coefficient and show how it interacts with coverability to derive online regret bounds. The regret bound is shown to be tight with a matching lower bound.
These theoretical insights motivate our algorithm, routed ensemble policy learning, which learns a small set of basis policies and routes trajectories to the policies that give the highest reward.
Lastly, to interpret objective conflict and routing behavior, we develop a kernel-surrogate trajectory attribution algorithm that quantifies how individual trajectories influence policy learning across preference weights.

\subsection{Examples and motivation}\label{subsec:motivation}

We start with two common MOMDP use cases to motivate our theoretical framework.

\begin{example}[Learning the Pareto frontier of a MOMDP]\label{example_pareto}
    The {Pareto frontier} consists of policies for which no objective can be improved without worsening at least one other objective.  Formally, a policy $\pi$ {dominates} $\pi'$ (denoted $\pi \succ \pi'$) if $J_i(\pi) \ge J_i(\pi')$ for all objectives $i \in [m]$, with strict inequality for at least one $i$, where $J_i(\pi)$ is the $i$-th entry of $\mathbf{J}(\pi)$. 
\end{example}

\begin{figure}[t!]
    \centering
    \begin{subfigure}[b]{0.99\linewidth}
        \centering
        \includegraphics[width=0.999\linewidth]{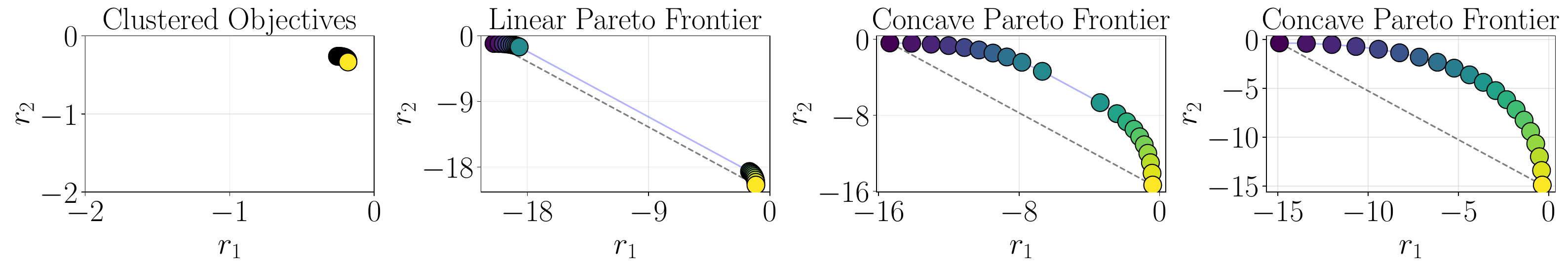}
    \end{subfigure}
    \caption{We visualize the objective space returns $(r_1, r_2)$ for policies optimized under varying preference weights $w \in [0, 1]$ in three canonical regimes --- clustered, linear and concave Pareto frontier. Light yellow indicates that $w$ is closer to $1$, while dark blue indicates $w$ is closer to $0$.
    }\label{fig_motivation}
\end{figure}

\begin{example}[Online exploration under corrupted objectives]\label{example_adversarial}
    Consider an MOMDP with $m = 2$, where $r_1$ is the clean task reward and $r_2$ may be corrupted (possibly adversarial). %
    The target preference is $w = [1,0]^\top$, so the goal is to optimize $r_1$. However, if corruption is present, trajectories collected online may be influenced by the spurious signal in $r_2$. This motivates the need to {identify} corrupted trajectories and to {attribute} their alignment with $w$.
\end{example}

We illustrate that learning in an MOMDP depends on how objectives relate to one another. In the two-objective case ($m=2$), there are three canonical regimes, as shown in Figure~\ref{fig_motivation}: (i) \emph{aligned objectives}, (ii) \emph{conflicting objectives with an (approximately) linear Pareto frontier}, and (iii) \emph{conflicting objectives with a strictly concave frontier}. Aligned objectives admit a single policy that is optimal for all preferences. Linear conflict is captured by two corner policies whose mixtures span the frontier. Concave conflict requires additional intermediate policies beyond the corner solutions. When $m$ grows, different objectives can simultaneously fall into different regimes—some aligned, linearly conflicting, and others exhibiting concave trade-offs. These induce discontinuous shifts and make policy optimization harder.

\subsection{Theoretical characterization through conflict coefficients}\label{sec_coverability}

With the above examples in mind, next, we quantify how objective conflict impacts sample complexity in MORL. 
We assume that there exists a set of $K$ latent basis policies that (approximately) represent optimal policies across preferences for $w \in \Omega$. The number of basis policies $K$ captures the intrinsic complexity of the objective trade-offs. %
We capture the extent of conflict via the \emph{objective conflict coefficient} $\Gamma$, which measures the diversity of basis occupancies relative to their mixture.
\begin{definition}\label{def:obj_conflict_coef}
    Let the basis-mixture occupancy be the uniform average of the $K$ basis policies, $\overline{d}_h(s,a) \coloneqq \frac{1}{K} \sum_{k=1}^{K} d_h^{\pi^\ast_{\basis k}}(s,a)$. 
    The \textit{objective conflict coefficient} is 
    \begin{align} \Gamma \coloneqq \max_{h \in [H]} \max_{k \in [K]} \chi^2\left(d_h^{\pi^\ast_{\basis k}} \| \overline{d}_h\right). \label{eq_gamma}\end{align}
\end{definition}
Smaller $\Gamma$ indicates substantial overlap among basis policies and easier transfer across preferences, while large values correspond to near-disjoint occupancies and greater learning difficulty. Moreover, $\Gamma$ is always bounded by $K-1$ (Lemma~\ref{lem:gamma_K_ub}), so its magnitude scales with the number of basis policies. For instance, $\Gamma=0$ in Case~1, and $\Gamma$ is larger in Case~3 than in Case~2.

We now analyze online MORL through cumulative regret:
\[\mathsf{Reg}(T) = \sum_{t=1}^T \left( \mathcal{J}(\pi^\star_{w_t}; w_t)  - \mathcal{J}(\pi^{(t)}; w_t) \right) \,. \]

We impose two structural assumptions. The first states that optimal policies lie in the convex hull of the basis occupancies.

\begin{assumption}%
\label{assm:ref_completeness}
    There exists a finite set of latent basis policies $\{\pi^\star_{ 1},\ldots,\pi^\star_{ K}\}$ such that for any $w\in\Omega$, there exists a mixing vector $\nu_w\in \Delta^{K-1}$ (the probability simplex over $K$ elements) satisfying $d_h^{\pi^\star_w}(s,a) = \sum_{k=1}^{K} \nu_w(k)\, d_h^{\pi^\star_{\basis k}}(s,a)$ for all $ (s,a)\in\mathcal{S}\times\mathcal{A}\,.$
\end{assumption}

This assumption is imposed to yield a clean characterization of the role of $K$ in regret.
We relax this condition in Appendix~\ref{app:approx-realizability} by allowing optimal occupancies to lie within an $\epsilon$-neighborhood of the convex hull, under which the same qualitative conclusions continue to hold. 
The second assumption ensures sufficient exploration of all basis components.

\begin{assumption}%
\label{assm:mixture_stability}
    For each $t\ge 1$, there exists $\nu^{(t)}\in [0,1]^K$ such that $d_h^{\pi^{(t)}}(s,a)=\sum_{k=1}^{K} \nu^{(t)}(k)\, d_h^{\pi^\star_{\basis k}}(s,a)$ for all $(s,a)$ and $H$. 
    Letting $\overline{\nu}^{(t)}(k) = \frac{1}{t-1}\sum_{i=1}^{t-1}\nu^{(i)}(k)$, we assume there exist $c_0\in(0,1]$ and a burn-in time $t_0$ such that $\min_{k\in[K]} \overline{\nu}^{(t)}(k) \ge {c_0}/{(C_{\cov}K)}$ for all $t\ge t_0$.
\end{assumption}

This condition holds, for example, when preferences are sampled with full support over $\Omega$, ensuring persistent coverage of all basis policies.

We now derive a regret bound that explicitly quantifies the impact of $K$ and $\Gamma$. We study a $K$-policy variant of the optimistic confidence-set algorithm~\cite{jin2021bellman}. The learner maintains $K$ confidence sets $\cF_k^{(t-1)}(\beta)$, one for each basis index $k\in[K]$, and at episode $t$ executes the optimistic policy $\pi^{(t)}=\pi_{f^{(t)}}$ associated with some $f^{(t)}\in\cF^{(t-1)}_{k_t}(\beta)$ for the routed basis $k_t$. A detailed description is provided in Appendix~\ref{app:proof}.
Each $\cF_k$ is given by $\cF_{k,1}\times\cdots\times\cF_{k,H}$, with $\cF_{k,h}\subseteq(\cS\times\cA\times\Omega\to[0,H])$. We assume that $\cF_k$ is finite and write $|\cF| =\max_{k\in[K]}|\cF_k|$. For infinite classes, one can replace $\log|\cF|$ via covering numbers \cite{jin2021bellman}.
Our main result is stated as follows.

\begin{theorem}%
    \label{theorem:regret}
    Suppose scalarized rewards are bounded in $[0,1]$ per step, Bellman realizability and completeness hold for each per-basis class $\cF_k$, and Assumptions~\ref{assm:ref_completeness}--\ref{assm:mixture_stability} hold. Let $\beta = c_1\log(K|\cF|T/\delta)$ be the confidence-set radius after a union bound across the $K$ basis classes. After $t_0=\mathcal{O}(C_\cov K\beta)$, then with probability at least $1-\delta$, there exists an algorithm with $K$ policies such that the cumulative regret satisfies
	     \begin{equation}\label{eq:regret_explicit}
	    \mathsf{Reg}(T) \le {\mathcal{O}_T}\left( H \sqrt{C_\cov(1 + \Gamma) \log(K \abs{\cF} T \delta^{-1}) T} \right). %
	\end{equation}
\end{theorem}

Theorem~\ref{theorem:regret} decomposes learning difficulty into (i) exploration hardness, governed by $C_{\cov}$, and (ii) transfer hardness induced by objective conflict, captured by $\Gamma$. When objectives are well aligned, both $\Gamma$ and $K$ are small, and the regret approaches the single-objective rate; when basis occupancies are nearly disjoint, $\Gamma$ increases, reflecting the cost of transferring information across preferences.  The full proof of Theorem \ref{theorem:regret} is given in Appendix~\ref{app:proof}.

\textbf{Nearly-tight regret lower bound.} We complement this upper bound with \emph{a matching lower bound} in Theorem \ref{theorem:lower_bound} in Appendix~\ref{app:lower_bound}, in the order of $\Omega_T\big( H\sqrt{(1+\Gamma)T}\big)$, establishing the tightness of the dependence on $\Gamma$ and $T$ up to $C_{\cov}$ and logarithmic factors.

\smallskip
\begin{remark}
    Although Theorem \ref{theorem:regret} is developed for a stylized online algorithm, similar proof techniques can be extended to more general settings.
    Proposition~\ref{prop:approx-regret} in Appendix~\ref{app:approx-realizability} shows that, when the low-rank basis Assumption~\ref{assm:ref_completeness} is relaxed to approximate low-rankness, the resulting regret bound preserves the same qualitative dependence on $K$ and $\Gamma$.
    
    Furthermore, empirical analysis in Appendix~\ref{app_omitted} demonstrates that these insights extend beyond the theoretical setting.
    To assess the approximation empirically, we measure the convex-hull residual over a preference grid. In a tabular concave setting, the mean residual decreases from $0.186$ at $K=3$ to $1.45\times10^{-4}$ at $K=21$; on MO-Hopper, it decreases from $0.443$ at $K=3$ to $0.231$ at $K=15$ using candidate-selected bases. The analysis below uses the exact assumption, while Appendix~\ref{app:approx-realizability} treats nonzero residuals.
    We also find that increasing the size of $N$ gradually reduces the approximation error imposed by the low-rank condition (See Table \ref{tab:occupancy_hull_residuals} for the evaluation).
\end{remark}

\subsection{Learning routed ensemble policies}

In Sections~\ref{subsec:motivation} and~\ref{sec_coverability}, the basis policies are assumed latent. Next, we design a scalable \emph{routed ensemble policy learning} algorithm that:
\begin{itemize}
    \item[(i)] explicitly learn a set of $N$ basis policies, with each trajectory routed to the highest-reward basis, and 
    \item[(ii)] learn a preference-conditioned transfer operator $f_w: \Pi^N \rightarrow \Pi$ that maps these basis policies to a composite policy for any given preference $w$.
\end{itemize}
We also clarify that the ensemble size $N$ may or may not be the same as $K$ as stated in Assumption \ref{assm:ref_completeness}.
Thus, the algorithms studied will apply to a real-world MORL environment.

Our proposed Algorithm~\ref{alg_routing} alternates between scoring, routing, and updating:
(i) \emph{Scoring}: For each trajectory and each basis policy, compute a score using its estimated scalarized return (cf.~Equation~\eqref{eq_score} below).
(ii) \emph{Routing}: Assign each trajectory to the basis policy that achieves the highest score for the corresponding preference, inducing a partition of the data into $N$ clusters.
(iii) \emph{Updating}: Update each basis policy using only the trajectories assigned to its cluster via a standard RL subroutine (e.g., upper-confidence bound value iteration for MOMDPs~\cite{wu2021accommodating}).
(iv) \emph{Iteration}: Repeat the above steps across online batches until both the cluster assignments and the learned policies stabilize.
We formalize the learning objective as:
\begin{align}\label{eq_opt}
    \max_{\pi_1, \pi_2, \dots, \pi_N} \exarg{s_0 \sim \rho, w \sim \cP}{ \biginner{w}{\sum_{h=0}^H \mathbf{r}(s_h, f_w)} }\,,
\end{align}
which maximizes the expected scalarized return $\mathcal{J}(\pi; w)$ of the composed policy $f_w(\pi_1,\ldots,\pi_N)$ under the target preference distribution $\mathcal{P}$.

Importantly, the routing step is \emph{optimal} under this formulation: when $f_w$ lies in the convex hull of a fixed set of basis policies, the optimal solution reduces to deterministically selecting the basis policy that maximizes the return under $w$. This follows from the linearity of $\mathcal{J}(\pi; w)$ in $w$, and is formalized in the following lemma.

\begin{lemma}\label{lem_routing}
    Fix a set of $N$ policies $\{\pi_1,\ldots,\pi_N\}$.
    Suppose $f_w \in \mathrm{conv}\{\pi_1,\ldots,\pi_N\}$ and $d_h^{f_w}(s,a) = \sum_{i=1}^N \alpha_i(w) d_h^{\pi_i}(s,a)$ for $\alpha(w) \in \Delta^{N-1}$.
    There exists an optimal routing function $\widehat f$ for Equation~\eqref{eq_opt} such that, for every $w \in \Omega$, $\widehat{f}_w = \pi_{j}$, where $j \in \arg\max_{i \in \set{1, 2, \dots, N}} \langle w, J(\pi_i) \rangle.$
\end{lemma}

\begin{proof}
    For any $w$, since $f_w$ lies in the convex hull of $\pi_1, \dots, \pi_N$, the state-action occupancy probabilities satisfy $d^{f_w}_h(s, a) = \sum_i \alpha_i(w) d^{\pi_i}_h(s, a)$.
    Thus, the scalarized return satisfies
    \begin{align}\label{eq_score}
        J_\alpha(f_w) =& \sum_h \sum_{a_h, s_h} d_h^{f_w}(s_h,a_h) \biginner{w}{r_h(s,a)} \nonumber \\
         =& \sum_h \sum_{a_h, s_h} \sum_{i=1}^N \alpha_i(w) d_h^{\pi_i}(s_h,a_h) \biginner{w}{r_h(s,a)} = \sum_{i=1}^N \alpha_i(w)\biginner{J(\pi_i)}{w}.
    \end{align}
    The maximum over $\alpha \in \Delta^{N-1}$ is therefore attained at an extreme point, corresponding to selecting a single basis policy.
    Thus, the optimal strategy is to choose the policy $\pi_i$ that maximizes $\langle J(\pi_i), w\rangle$.
\end{proof}

Lemma \ref{lem_routing} motivates the scoring and routing steps: we compute scores using equation~\eqref{eq_score} and route each trajectory to the basis policy that maximizes the score. Note that in the updating step, each basis policy is trained only on its assigned trajectories. This is crucial because a policy that performs well for one preference $w$ may perform poorly for a different $w'$. Consequently, mixing trajectories across preferences can degrade learning~\cite{wu2021accommodating}. Moreover, Proposition~\ref{prop:offpolicy-conflict} in Appendix~\ref{subsec:algorithm-3-analysis} connects this algorithm to $\Gamma$, showing that the intuition underlying Theorem~\ref{theorem:regret} carries over to its performance.

Finally, while our algorithm assumes $N$ is given, in practice it can be selected via cross-validation (Section~\ref{sec_abla}). One may adaptively grow $N$ by identifying new clusters when existing policies do not explain incoming trajectories, drawing on the clustering literature \cite{li2021learning, chen2018practical}.
Compared to previous work \cite{wu2021accommodating} that learns one optimal policy for each objective, the routed ensemble algorithm can leverage potential overlap that is latent in the underlying RL environment.

\begin{algorithm}[t!]
\caption{Learning Routed Ensemble Policies for Multi-Objective Markov Decision Processes}\label{alg_routing}
\textbf{Input:} A MOMDP $\mathcal{M}$; Initial state distribution $\rho$; Preference distribution $\cP$; Number of policies $N$\\
\textbf{Require:} A policy learning algorithm $F$; A policy evaluation function $J$;  Number of iterations $T$; Horizon $H$; Number of observed trajectories $B$ per iteration\\ %
\textbf{Output:} A set of $N$ policies %
\begin{algorithmic}[1]
    \State $\pi_1^{(0)}, \pi_2^{(0)}, \dots, \pi_N^{(0)}: \cS \rightarrow \Delta(\cA)$, start from $N$ random initializations
    \For{$t = 1, \dots, T$}
        \State $w_t \sim \cP$, $\mathcal{D}_t \leftarrow \emptyset$, $S\leftarrow$ $N$ by $B$ score matrix
        \For{$i = 1, \dots, B$}
            \State $j_i \leftarrow$ sample uniformly from $\set{1, 2, \dots, N}$
            \State $s_0 \leftarrow \rho$: sample from the initial state distribution
            \State $\tau^{(t)}_{i} \leftarrow (s_0, a_0, \ldots, s_H, a_H)$: simulate from $s_0$ for $H$ steps using $\pi_{j_i}^{(t-1)}$ %
            \For{$j = 1, \dots, N$}
                \State $S_{j, i} \leftarrow \sum\limits_{h=0}^H \Big\langle{{w}_t, \sum\limits_{a\in\cA} \mathbf{r}(s_h, a) \pi_j^{(t-1)}(s_h, a)} \Big\rangle$  \Comment{Score policy $j$ on trajectory $i$}
            \EndFor
            \State $\mathcal{D}_t \leftarrow \mathcal{D}_t \cup \{ (\tau^{(t)}_{i}, w_t) \}$
        \EndFor
        \For{$j = 1, \dots, N$}
           \State $\mathcal{D}_{t,j} \leftarrow \Big\{(\tau^{(t)}_{i}, w_t) \in \mathcal{D}_t \mid j \in \argmax\limits_{1\le j'\le N} S_{j', i}\Big\}$ \Comment{Route trajectories to policy $j$}
           \State $\pi_j^{(t)} \leftarrow$ $F\big(\pi_j^{(t - 1)}, \mathcal{D}_{t,j}\big)$ \Comment{Update policy $j$ using algorithm $F$}
       \EndFor
    \EndFor
    \State \textbf{return} $\pi_1^{(T)}, \pi_2^{(T)}, \dots, \pi_N^{(T)}$
\end{algorithmic}
\end{algorithm}

\subsection{Trajectory attribution}\label{sec:attribution}

Another approach to address the problem of objective conflict is through trajectory selection, for instance, removing trajectories that are far from the learned policy \cite{zhang2025scalable}.
This motivates the question of quantifying how individual training trajectories contribute to policy performance, a problem that we term \emph{trajectory attribution}. This complements our theoretical analysis above by providing a data-driven measure of which trajectories help or hinder learning and by improving the interpretability of routed ensemble policies. %

A natural starting point for tackling trajectory attribution is via leave-one-out (LOO) retraining, which measures the change in performance when trajectory $t$ is excluded:
\[ \mathcal{I}_t^{\text{LOO}} = \cV\big([\mathbf{1}]_T\big) - \cV\big([\mathbf{1}]_T - e_t\big), \]
where, for any subset indicator $S \in \{0,1\}^T$, we define $\mathcal{V}(S):=\mathbb{E}_{w\sim \cP}\left[\mathcal{J}\big(\pi_S; w\big)\right],\pi_S := F(\mathcal{D}_S)$, with $\mathcal{D}_S := \{(\tau^{(j)},w_j)\in \mathcal{D}: S_j=1\}.$
Here $\pi_S$ is the policy trained on the subset selected by $S$ using algorithm $F$, $[\mathbf{1}]_T$ is the all-ones vector, and $e_t$ is the $t$-th standard basis vector.

Thus, $\mathcal{I}_t^{\text{LOO}}$ reflects the marginal importance of trajectory $t$. However, exact LOO retraining becomes computationally infeasible as the number of trajectories grows \cite{liidentification}. Moreover, it only captures one-at-a-time effects and misses nonlinear interactions among trajectories.

To address these limitations, Algorithm~\ref{alg_krr} introduces a scalable \emph{kernel-surrogate} approach \cite{liidentification,zhang2026efficient} that approximates the subset-performance function $\mathcal{V}(\cdot)$. This procedure proceeds in three steps. First, we randomly sample $n$ trajectory subsets $S_1,\ldots,S_n$, where each $S_i\in\{0,1\}^T$ indicates which trajectories are included. Second, for each subset $S_i$, we estimate its performance using a surrogate loss $v(S_i)\approx\mathcal{V}(S_i)$. Third, we fit a kernel ridge regression surrogate: we form the radial basis function kernel $\Psi_{ij} =\exp(-\gamma\|S_i-S_j\|^2)$, set $V=[\widehat v(S_1),\ldots,\widehat v(S_n)]^\top$, and compute $\alpha=(\Psi+\lambda \id_n)^{-1}V$. This yields the predictor \[ g_\alpha(S)=\sum_{i=1}^n\alpha_i\exp(-\gamma\|S_i-S\|^2), \text{ for } \mathcal{V}(S). \] Because this kernel model $g_\alpha(S)$ is nonlinear in the subset indicator, it can capture higher-order interactions such as synergy and antagonism among trajectories.
Provided with $g_\alpha(S)$, we could then use this to score each trajectory; for instance, by averaging over all the subsets that include a particular trajectory.
Then, trajectories with the highest scores are selected for downstream applications.

The procedure above is formulated for trajectory attribution. However, the same procedure can be applied to objective attribution by sampling subsets of objectives. Each subset retains only the training contributions associated with its selected objectives. Removing an objective therefore excludes its contribution across all trajectories and allows us to estimate its effect on policy performance.

\paragraph{Implementation.} In online RL, exact retraining for each sampled subset is often impractical because the training data evolve. We therefore estimate the subset performance $v(S_i)$ locally around a reference policy trained on all trajectories. Specifically, we use policy gradients from local rollouts to approximate the parameter update induced by each subset, avoiding the cost of full retraining. The algorithm and the corresponding derivation are presented in Appendix \ref{sec:algorithm}.

The experimental results support this approximation. In Figure~\ref{fig_attr_loo}, Appendix \ref{app_omitted}, we show that the fitted kernel model $g_\alpha([\mathbf{1}]_T-e_t)$ is strongly positively correlated with the LOO scores $\mathcal{I}_t^{\mathrm{LOO}}$ and strongly negatively correlated with the objective conflict coefficient $\Gamma$. Moreover, Algorithm~\ref{alg_krr} achieves a Spearman rank correlation of $0.75$ with the true policy-gradient influence scores, confirming that it identifies influential trajectories.
More details of the approximation procedure are presented in Appendix \ref{sec:algorithm}.

\begin{algorithm}[t!]
\caption{Kernel Surrogates for Trajectory Attribution Using Policy Gradient Estimation}\label{alg_krr}
\textbf{Input:} A multi-objective RL environment with $m$ reward objectives\\
\textbf{Requires:} Subset sampling distribution $\cD$; Number of subsets $n$; $\ell_2$-regularization parameter $\lambda$; A policy learning algorithm $F$\\
\textbf{Output:} A kernel surrogate model $g_\theta: \set{0, 1}^m \rightarrow \real$
\begin{algorithmic}[1]
    \State $\theta^\star \leftarrow$ Learn a meta policy with all the $m$ reward objectives simultaneously using $F$
    \State ${S}_1, \dots, {S}_n \leftarrow$ Sample $n$ subsets from $\cD$ independently
    \State $\widehat{v}({S}_1), \widehat v(S_2), \dots, \widehat v(S_n) \leftarrow$ Apply \acronym{} (Algorithm \ref{alg_1}) with $(S_1, S_2, \dots, S_n; \theta^\star)$
    \State $\Psi_{ij} \leftarrow \exp(-\gamma \bignorm{{S}_{i} - {S}_{j}}^2)$, for all $1\le i, j \le n$ %
    \Comment{Construct kernel matrix} 
    \State ${V} \leftarrow [\widehat v(S_{1}), \ldots, \widehat v(S_{n})]^{\top}$%
    \Comment{Collect estimated outcomes}
    \State $\alpha \leftarrow (\Psi + \lambda \id_n)^{-1} V$%
    \Comment{Fit kernel ridge regression}
    \State \textbf{return} $g_\theta({S}) = \sum_{i=1}^n \alpha_i \exp(-\gamma\|S_i-S\|^2)$
\end{algorithmic}
\end{algorithm}
\section{Experiments}\label{sec:experiment}

Our experiments are designed to address the following two aspects: (1) to verify the theoretical insights from the regret analysis, and (2) to evaluate the empirical performance of our algorithms across multi-objective RL settings, including synthetic environments, continuous robotic control, and RLHF settings.

We highlight three findings. First, routing mitigates negative transfer and reduces cumulative regret in synthetic environments. Second, routed ensembles outperform baselines by a 4.3\% relative margin in multi-task RL, up to 48\% in multi-objective RL, and 17\% in large language model benchmarks. Third, our kernel surrogates improve correlation with ground-truth marginal gains by 50\% for trajectory attribution and improve downstream task-selection reward by 27\%. Our code repository can be found at \href{https://github.com/VirtuosoResearch/Multi-objective-reinforcement-learning-code/}{https://github.com/VirtuosoResearch/Multi-objective-reinforcement-learning-code/}.%

\subsection{Evaluation in synthetic environments}

We validate our theoretical insights in the following two environments (details can be found in Appendix \ref{app_syn_env}).

\emph{1. Learning Pareto frontier (Example~\ref{example_pareto}).} The environment has $m=4$ conflicting objectives. The reward functions are designed such that maximizing one objective necessitates sacrificing the others, creating a complex, high-dimensional Pareto frontier. We sample preference vectors $w$ uniformly from the simplex to evaluate the approximation of the frontier.

\noindent\emph{2. Learning under corrupted examples (Example~\ref{example_adversarial}).} The environment has $m=2$ objectives, where $r_1$ represents a clean task signal and $r_2$ represents a corrupted signal. The target preference is fixed at $[1, 0]$, meaning the agent solves a single-objective problem but is exposed to conflicting gradients from $r_2$ during exploration. The training preference is sampled uniformly in the simplex. 

\noindent\emph{3. Routing under varying goal separation.} We keep the environment dynamics fixed and vary the separation between goals. For each setting, we compute $\Gamma$ from the resulting occupancies and measure the reduction in cumulative regret achieved by routing relative to the comparison method.

For Pareto frontier learning, Figure~\ref{fig_momdp} shows that routed ensemble learning substantially reduces cumulative regret relative to MO-UCBVI~\cite{wu2021accommodating}, using one policy per objective. This suggests that routed ensembles maintain diverse basis policies that cover different regions of the frontier, while the single-policy baseline fails to capture the complex four-way trade-off.

For learning under corrupted examples, Figure~\ref{fig_admdp} shows that our approach mitigates negative transfer from corrupted inputs and achieves near-optimal performance. This suggests that routing can filter out harmful interference and assign trajectories to support optimization of the clean objective $r_1$.

\subsection{Evaluation on robotic control and language models}

Next, we present our evaluations applied to several robotic control and language modeling settings. We first describe the setups.

\paragraph{Environments and models.} We evaluate on four regimes, each with a different source of conflict: (i) \emph{shared dynamics} on the 10-task MT10 benchmark from Meta-World~\cite{yu2020metaworld}; (ii) \emph{explicit Pareto trade-offs} on the continuous MO-HalfCheetah/Hopper/Ant and discrete Deep-Sea-Treasure/Fruit-Tree environments of MO-Gymnasium~\cite{felten_toolkit_2023}; (iii) \emph{preference-induced conflict} on HelpSteer2~\cite{wang2024helpsteer}, with Qwen3-1.7B fine-tuned with PPO via per-basis LoRA adapters; and (iv) \emph{trajectory attribution} on Multitask Minigrid, MT10, and a synthetic synergy MTRL environment (See Appendix~\ref{app_syn_env}). In all regimes, routing is applied as a drop-in module on top of the published base learner.

\begin{figure}[!t]
  \centering
  \begin{subfigure}{0.49\linewidth}
    \centering
    \includegraphics[width=0.95\linewidth]{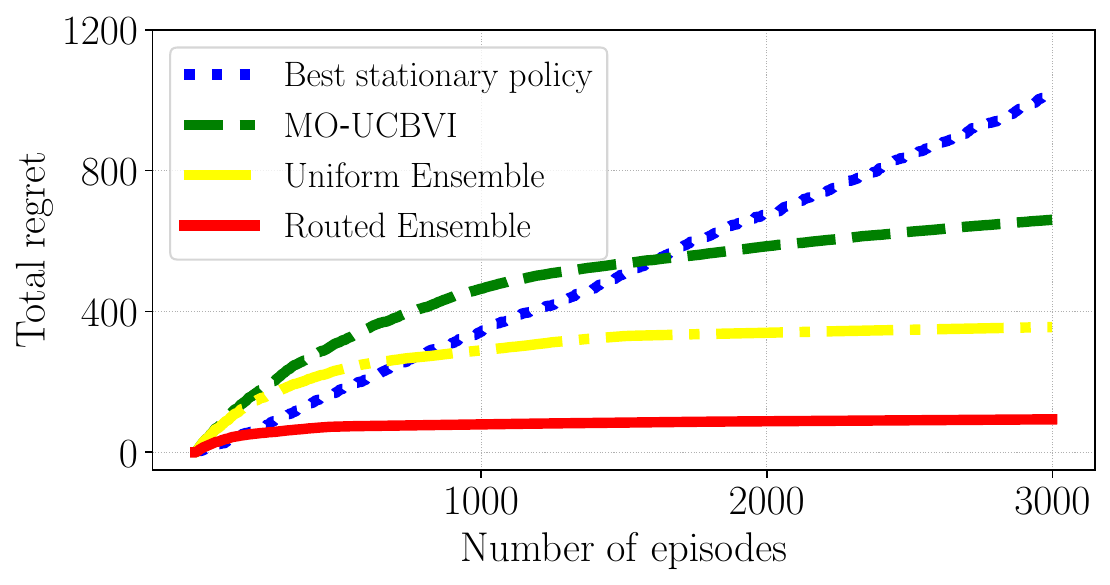}
    \caption{Learning Pareto frontier}
    \label{fig_momdp}
  \end{subfigure}%
  \hfill %
  \begin{subfigure}{0.49\linewidth}
    \centering
    \includegraphics[width=0.95\linewidth]{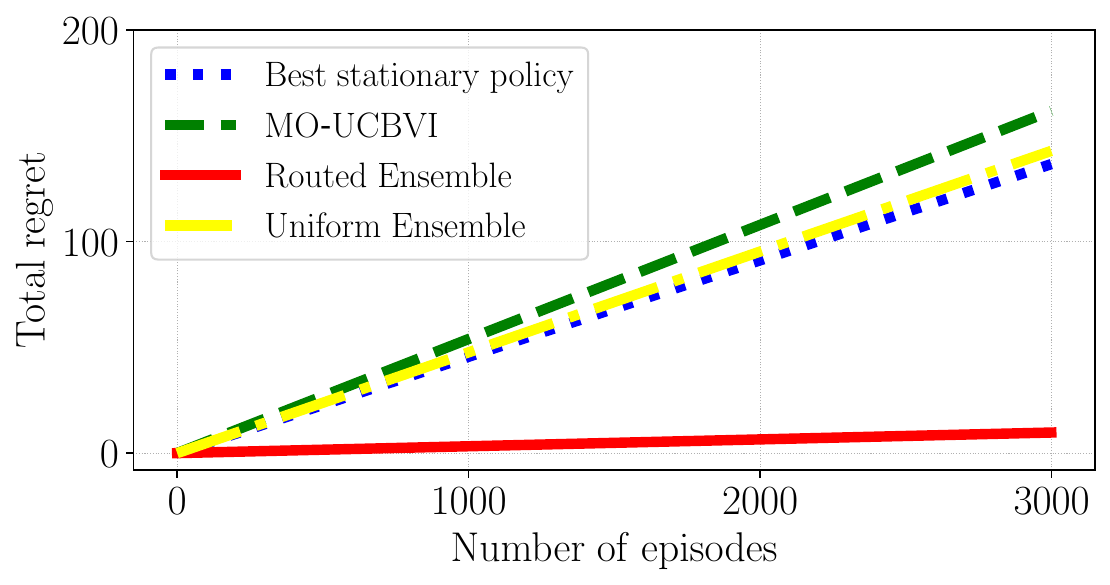}
    \caption{Learning under corrupted examples}
    \label{fig_admdp}
  \end{subfigure}
  \caption{We compare the cumulative regret of routed ensembles (red) against MO-UCBVI
  (green dashed), uniform ensembles (yellow), and the best stationary policy (blue dotted).}\label{fig:momdp-two}
\end{figure}

\paragraph{Baselines.} We compare to the following baselines: (i) For multi-task RL, Soft Modularization~\cite{yang2020multi}, PaCo~\cite{sun2022paco}, CARE~\cite{sodhani2021multi}; (ii) For continuous multi-objective RL, PGMORL~\cite{xu2020prediction}, Q-Pensieve~\cite{hung2022q}, CAPQL~\cite{lu2023multi}, GPI-LS~\cite{alegre2023sample}, and learned-gate MoE; for discrete multi-objective RL, PQL~\cite{van2014multi}, MPMOQL~\cite{van2014multi}, and Envelop~\cite{yang2019generalized}; (iii) For RLHF, a single-policy model; (iv) For trajectory attribution, Influence Functions~\cite{koh2017understanding}, TracIn~\cite{pruthi2020estimating} adapted to online RL setting~\cite{hu2025snapshot}, and linear surrogate models~\cite{liidentification}. The MoE and GPI-LS comparisons hold the base policy-learning algorithm and interaction budget fixed to examine different routing and policy-selection mechanisms. For the first three regimes, we also include \emph{uniform assignment}, which randomly assigns trajectories to basis policies.

\paragraph{Evaluation.} We use the following evaluation metrics: (i) success rate on MT10; (ii) hypervolume (HV) and utility (UT) on continuous MO-Gymnasium, and HV and cardinality (Card.) on discrete MO-Gymnasium; (iii) preference-induced conflict on HelpSteer2~\cite{wang2024helpsteer} and UltraFeedback~\cite{cui2023ultrafeedback}, using Qwen3-1.7B and Llama-3.2-1B fine-tuned with PPO via per-basis LoRA adapters; (iv) for trajectory attribution, linear data modeling score (LDS) and selection performance. LDS is the Spearman rank correlation between predicted and ground-truth marginal gains. Selection performance is the target objective return on a target objective when training within the top-$3$ most influential source tasks selected by each procedure.

\paragraph{Results for policy optimization.} Our proposed routed ensemble Algorithm \ref{alg_routing} consistently outperforms baselines in multi-task and multi-objective settings. The comparison is presented in Table~\ref{tab_mtrl_results}. On MT10, it achieves an average success rate of $85.6\%$, outperforming soft modularization ($82.0\%$) by a relative margin of $4.3\%$. %

Our routing algorithm also improves RLHF performance on both the synthetic environment and real-data benchmarks.
In Table \ref{tab_llm}, we compare Algorithm \ref{alg_routing}, uniform assignment, and learning a single policy for each objective.
Our algorithm consistently increases hypervolume on the induced frontier. In the synthetic environment, our method improves hypervolume by $17\%$ over the single-policy baseline and $14\%$ over the uniform ensemble. Averaged across the two models, the corresponding improvements are $5.0\%$ and $8.2\%$ on HelpSteer2, and $6.2\%$ and $3.0\%$ on UltraFeedback.

Next, we present comparisons in continuous MO-Gymnasium environments in Table~\ref{tab:routing_baselines}. Our routing algorithm yields an $8\%$ average improvement over Q-Pensieve and achieves higher HV than learned-gate MoE and GPI-LS on all three benchmarks. The capacity-matched MoE comparison supports the benefit of trajectory routing beyond additional expert capacity.
These results suggest that routing trajectories across basis policies helps distinguish aligned from conflicting data and approximates the Pareto frontier more efficiently than uniform assignment.

\paragraph{Results for trajectory attribution.} By applying Algorithm \ref{alg_krr} at the objective level, we can achieve higher LDS scores and downstream task-selection performance. The results are presented in Table~\ref{tab_task_selection_reward}. The kernel surrogate improves LDS by $50\%$ over attribution baselines across Multitask-Minigrid and MT10. Moreover, when used to select source tasks for co-training, kernel-surrogate attribution improves the target reward by $27\%$, demonstrating that the learned attributions are useful for downstream policy learning.

\subsection{Ablation studies}\label{sec_abla}

\begin{table*}[t!]
\centering
\caption{Comparison of our approach with several baseline methods. We report the average success rate on Meta-World, HV and UT on MuJoCo, as well as HV and Cardinality on Grid-World.}\label{tab_mtrl_results}%
\resizebox{\textwidth}{!}{\small\begin{tabular}{l|c|ccccc}
\toprule
 & Metrics & Soft Modularization  & PaCo  & CARE & Uniform Assignment & \bf{Algorithm \ref{alg_routing}} \\
\midrule
Meta-World & Success rate &  $82.0_{\pm1.1}\%$ & $73.1_{\pm1.1}\%$ &  $80.2_{\pm1.8}\%$ & $81.3_{\pm1.4}\%$ & $\bm{85.6}_{\pm1.8}\%$ \\ 
\midrule
                    & Metrics & PGMORL & Q-Pensieve & CAPQL & Uniform Assignment & \bf{Algorithm \ref{alg_routing}} \\
\midrule                    
\multirow{2}{*}{MO-HalfCheetah2d} & HV ($\times 10^6$)  & $0.11_{\pm0.08}$ & $4.85_{\pm0.06}$ & $3.29_{\pm0.18}$ & $3.17_{\pm0.11}$ & $\bm{5.17}_{\pm0.16}$ \\
                               & UT ($\times10^2$)      & $0.12_{\pm0.13}$ & $2.14_{\pm0.13}$ & $0.63_{\pm0.09}$ & $1.01_{\pm0.12}$ & $\bm{2.47}_{\pm0.14}$ \\ 
\multirow{2}{*}{MO-Hopper}     &  HV ($\times 10^7$) & $1.23_{\pm0.04}$ & $1.61_{\pm0.32}$ & $1.39_{\pm0.13}$ & $1.56_{\pm0.16}$ & $\bm{1.83}_{\pm0.14}$\\
                               & UT ($\times10^2$) & $0.04_{\pm0.02}$ & $1.47_{\pm0.14}$ & $1.29_{\pm0.12}$ & $1.17_{\pm0.09}$ & $\bm{1.56}_{\pm0.06}$ \\ 
\multirow{2}{*}{MO-Ant}      & HV ($\times 10^7$) & $0.17_{\pm0.11}$ & $1.83_{\pm0.28}$ & $0.84_{\pm0.05}$ & $1.81_{\pm0.15}$ & $\bm{2.03}_{\pm0.16}$\\
                               & UT ($\times10^2$) & $0.21_{\pm0.06}$ & $1.75_{\pm0.02}$& $1.76_{\pm0.16}$ & $1.68_{\pm0.08}$ & $\bm{1.82}_{\pm0.12}$ \\
\midrule
& Metrics & PQL & MPMOQL & Envelop & Uniform Assignment & \bf{Algorithm \ref{alg_routing}} \\
\midrule
\multirow{2}{*}{Deep-Sea-Treasure} & HV ($\times 10^3$) & $0.76_{\pm0.30}$ & $0.25_{\pm0.01}$ & $0.79_{\pm0.00}$ & $0.62_{\pm0.20}$ & $\bm{1.10}_{\pm0.00}$ \\
 & Card. ($\times 10$) & $0.83_{\pm0.15}$ & $0.17_{\pm0.06}$ & $0.10_{\pm0.00}$ & $0.80_{\pm0.11}$ & $\bm{1.00}_{\pm0.00}$ \\ 
\multirow{2}{*}{Fruit-Tree} & HV ($\times 10^3$) & $6.69_{\pm0.95}$ & $1.01_{\pm0.24}$ & $3.90_{\pm0.00}$ & $8.29_{\pm0.38}$ & $\bm{9.30}_{\pm0.00}$ \\
 & Card. ($\times 10$) & $3.40_{\pm0.36}$ & $0.17_{\pm0.06}$ & $3.00_{\pm0.00}$ & $5.47_{\pm0.06}$ & $\bm{6.40}_{\pm0.00}$ \\
\bottomrule
\end{tabular}}
\end{table*}

\begin{table}[t]%
    \centering
    \caption{Comparison results for RLHF settings between our approach, using a single policy, and uniform assignment to the same number of basis policies.}\label{tab_llm}\vspace{0.075in}
    {\small
    \begin{tabular}{lccccc}
        \toprule
        & \multirow{2}{*}{Synthetic}
        & \multicolumn{2}{c}{Qwen3-1.7B}
        & \multicolumn{2}{c}{Llama-3.2-1B} \\
        \cmidrule(lr){3-4}\cmidrule(lr){5-6}
        Method & & HelpSteer2 & UltraFeedback
        & HelpSteer2 & UltraFeedback \\
        \midrule
        Single Policy
        & 0.48 & 0.51 & 0.46 & 0.34 & 0.52 \\
        Uniform Assignment
        & 0.49 & 0.47 & 0.47 & 0.35 & 0.54 \\
        \textbf{Algorithm~\ref{alg_routing}}
        & \textbf{0.56} & \textbf{0.54} & \textbf{0.48}
        & \textbf{0.35} & \textbf{0.55} \\
        \bottomrule
    \end{tabular}}
\end{table}

\textbf{Ensemble policies reduce coverability.} We measure coverability using the union of state-action occupancies in the synthetic environment. A single policy results in a high $\chi^2$-divergence of $1.3$. The routed ensembles distribute exploration, reducing the metric to $0.2$. This suggests that routing mitigates the exploration bottleneck.

\textbf{Varying the number of basis policies in the ensemble.}
We vary $N$ in $\{2, 3, 4, 5\}$ to evaluate the sensitivity of Algorithm \ref{alg_routing} on three MuJoCo benchmarks.
For MO-HalfCheetah-v4, the Hypervolumes are highest at ${5.21} \times 10^6$ when $N= 3$.
For MO-Hopper, it is ${1.87}\times 10^7$ when $N = 3$.
For MO-Ant, it is ${2.12}\times 10^7$ when $N = 5$.
See Appendix \ref{app_omitted} for the illustration of these ablation results.

\textbf{Nonlinear kernels improve trajectory attribution.} We compare linear, quadratic, cubic, and RBF kernels by measuring the fitting loss of the kernel surrogate model in equation~\eqref{eq_surrogate_model_loss} on the synthetic dataset. The linear kernel converges to a loss of $0.2$, while the quadratic, cubic, and RBF kernels all achieve a lower loss of $0.1$. This gap demonstrates that the relationship between objective subsets and policy returns involves high-order interactions that a linear model cannot capture.

\textbf{From occupancy conflict to regret reduction.} Lastly, we evaluate how occupancy conflict relates to regret reduction via routing in Table~\ref{tab:goal_separation}. We vary the separation between two reward goals. As goal separation increases, measured $\Gamma$ rises from $0$ to $1$, while regret reduction from routing rises from $0$ to $36.03$. Across these settings, routing is more beneficial when preference-induced occupancies diverge.
This trend is consistent with Theorem~\ref{theorem:regret}, where larger occupancy conflict increases the conflict-dependent term in the regret bound.

\begin{figure}[t]
  \centering
  \begin{minipage}[t]{0.53\textwidth}
    \centering
    \captionsetup{type=table}
    \caption{Hypervolume comparison with learned-gate MoE and GPI-LS.}
    \label{tab:routing_baselines}
    \small
    \begin{tabular}{lccc}
      \toprule
      Method & HalfCheetah2d & Hopper & Ant \\
      \midrule
      MoE
      & $4.40_{\pm0.09}$ & $1.56_{\pm0.13}$ & $1.60_{\pm0.18}$ \\
      GPI-LS
      & $5.00_{\pm0.15}$ & $1.50_{\pm0.12}$ & $1.85_{\pm0.24}$ \\
      \textbf{Alg.~\ref{alg_routing}}
      & $\mathbf{5.17}_{\pm0.16}$
      & $\mathbf{1.83}_{\pm0.14}$
      & $\mathbf{2.03}_{\pm0.16}$ \\
      \bottomrule
    \end{tabular}
  \end{minipage}\hfill
  \begin{minipage}[t]{0.45\textwidth}
    \centering
    \captionsetup{type=table}
    \caption{Measuring conflict and regret reduction under varying
    goal separation levels.}
    \label{tab:goal_separation}
    \small
    \begin{tabular}{ccc}
      \toprule
      Separation & $\Gamma$ & Regret reduction \\
      \midrule
      0 & $0.00_{\pm0.00}$ & $0.00_{\pm0.00}$ \\
      2 & $0.74_{\pm0.00}$ & $5.77_{\pm0.11}$ \\
      8 & $1.00_{\pm0.00}$ & $36.03_{\pm0.57}$ \\
      \bottomrule
    \end{tabular}
  \end{minipage} 
\end{figure}

\begin{table}[t!]
    \caption{We report the average target reward when the target policy is trained with source tasks selected by each attribution algorithm.}\label{tab_task_selection_reward}\vspace{0.075in}
    {\small\begin{tabular}{lcc|cc|cc}
        \toprule
        \multirow{2}{*}{Methods} & \multicolumn{2}{c|}{Multitask Minigrid} & \multicolumn{2}{c|}{Meta-World MT10}        & \multicolumn{2}{c}{Synergy MTRL} \\
        & LDS & \multicolumn{1}{c|}{Reward} & LDS  & \multicolumn{1}{c|}{Reward} & LDS & \multicolumn{1}{c}{Reward} \\
        \midrule
        Influence functions & $-0.23_{\pm0.08}$ & $0.07_{\pm0.13}$  & $0.71_{\pm0.11}$ & $16.3_{\pm 1.4}$ & $0.04_{\pm0.16}$ & $1.77_{\pm0.67}$ \\
        TracIn & $0.31_{\pm0.06}$ & $0.25_{\pm0.21}$ & $0.45_{\pm0.15}$ & $16.0_{\pm 2.3}$ & $0.16_{\pm0.04}$ & $2.65_{\pm0.03}$ \\
        Linear surrogates & $0.37_{\pm0.12}$ & $0.33_{\pm0.19}$ & $0.42_{\pm0.19}$ & $15.8_{\pm 4.5}$ & $0.16_{\pm0.04}$ & $2.67_{\pm0.01}$\\
        \midrule
        \textbf{Algorithm \ref{alg_krr}} & $\mathbf{0.41}_{\pm0.04}$ & $\mathbf{0.45_{\pm 0.17}}$ & $\mathbf{0.80}_{\pm0.04}$ & $\mathbf{18.8_{\pm 5.6}}$ & $\mathbf{0.91_{\pm 0.01}}$ & $\mathbf{5.85_{\pm 0.01}}$ \\
        \bottomrule
    \end{tabular}}
\end{table}
\section{Related Work}\label{sec_related}

A reinforcement learning environment is typically formalized by means of a Markov decision process \cite{sutton1998reinforcement,agarwal2021theory}.
Existing approaches in the literature on multi-objective reinforcement learning (MORL) generally fall into two categories. The first reduces MORL to single-objective RL through scalarization, optimizing a fixed or sampled preference vector to learn a single policy \cite{gabor1998multi, miettinen2002scalarizing, mannor2004geometric, sodhani2021multi}. The second, which includes our work, examines the multi-objective nature of the problem by learning multiple policies (e.g., to approximate the Pareto frontier) to address sharp variation in optimal policies across preferences \cite{van2014multi,wu2021accommodating,liu2025pareto}. Recent work \cite{zhang2025scalable} proposes a convex relaxation algorithm for partitioning multiple objectives into disjoint, related groups. \emph{We build on this line of work by providing a theoretical characterization of objective conflicts in MORL, introducing a routed ensemble algorithm, and developing trajectory-level attribution to interpret model behavior}.
One interesting future direction is to examine the policy learning dynamics in multi-objective Markov decision processes \cite{kakade2001natural,lan2023policy}.
Lastly, sequential decision-making problems also occur in experimental design, where one often needs to roll out an experiment through staggered adoption \cite{xiong2024data,xiong2025automated}.

Recent work studies both offline and online RL with linear function approximation \cite{golowich2024role,golowich2024linear}. It would be interesting to understand the techniques therein and examine whether it would be possible to extend to learning optimal policies under multiple objectives under linear Bellman completeness.
More broadly, we refer the reader to a recent monograph by \citet{foster2023foundations} for references.

Influence estimation techniques \cite{koh2017understanding,pruthi2020estimating} have recently been extended from supervised learning to reinforcement learning to trace model behaviors. 
In the context of offline RL, \cite{deshmukh2023explaining} proposes a framework to attribute specific policy decisions to past trajectories.
A fundamental challenge in the online RL setting is the circular dependency where current policy updates influence future data distribution. Recent work \cite{hu2025snapshot} addresses this by employing local data attribution via gradient similarity to identify and filter harmful transitions.

\section{Conclusion}\label{sec_conclude}

In this paper, we study how to learn policies that balance competing objectives in multi-objective reinforcement learning. We first analyze the theoretical limits of transfer learning across preferences, highlighting that structural misalignment between objectives is the primary cause of negative transfer. To overcome this challenge, we propose a routed ensemble algorithm guided by a new measure, the objective conflict coefficient. Furthermore, we introduce a kernel-based trajectory-attribution method that captures widespread nonlinear interactions among trajectories and can improve subsequent policy learning. The experimental results demonstrate that our approach achieves superior sample efficiency and attribution accuracy compared to state-of-the-art baselines across diverse benchmarks.

One limitation of our theoretical framework is that the regret bounds do not directly apply to the routed ensemble algorithm. Further understanding the sample complexity of learning ensemble policies in multi-objective Markov decision processes is an interesting question for future work.

\section*{Acknowledgment}

We are grateful to the anonymous reviewers for their constructive comments that greatly improved this paper.
Thanks to Dongruo Zhou for pointing out references to prior work.
This work received financial assistance from NSF award IIS-2412008 and a startup fund from Northeastern University.
Zhenshuo Zhang received additional financial support from Northeastern University's Khoury PhD fellowships.
R. Xiong is supported in part by the URC award at Emory University.

\begin{refcontext}[sorting=nyt]
    \printbibliography
\end{refcontext}

\appendix

\newpage
\section{Policy Gradient Estimation}\label{sec:algorithm}

In Algorithm \ref{alg_krr}, we compute the kernel by minimizing the following objective over the training set:
\begin{equation}\label{eq_surrogate_model_loss} 
    \min_{g_\alpha \in \mathcal{H}} \sum_{i=1}^{n} \big( v({{S}}_{i}) - g_\alpha({{S}}_{i}) \big)^{2} + \lambda \lVert g_\alpha \rVert^{2}_\Psi,
\end{equation}
where $\lambda > 0$ is a regularization parameter and $\Psi$ is an $n \times n$ kernel matrix with entries $\Psi_{i,j} = \psi({{S}}_{i}, {{S}}_{j})$.
Since ${{S}}$ is binary, consider the radial basis function (RBF) kernel:
$\psi({{S}}_{i},{{S}}_{j}) = \exp\big(-\gamma \| {{S}}_{i} - {{S}}_{j} \|^{2}\big).$
By the representer theorem of kernel methods, the minimizer is: \[ g_\alpha({{S}})= \sum_{i=1}^{n} \alpha_i \psi\big({{S}}_{i}, {{S}}\big) \,, \]
where the coefficients $\alpha$ are obtained via $(\Psi + \lambda \id)^{-1}{V}$ with ${V} = [v({{S}}_{1}), \ldots, v({{S}}_{n})]^{\top}$, and $\id$ is the identity matrix. Below we describe how we obtain $v(S_i)$.

\paragraph{First-order approximation.}
We approximate the PPO surrogate objective (See Appendix \ref{app_ppo} for a detailed review of PPO) using a first-order Taylor expansion. Let $\varphi_t = \nabla_\theta \log \pi_\theta(a_t|s_t)|_{\theta=\theta^*}$ be the gradient feature vector. The policy objective is approximated as:
\[J(\theta) \approx \mathbb{E}_t \left[ \widehat{A}_t \cdot (1 + \varphi_t^\top \Delta \theta) \right], \]
where $\Delta \theta = \theta - \theta^*$ and $\widehat{A}_t$ is the advantage function.

Directly maximizing this linear approximation is unbounded. To obtain a robust estimator, we reformulate this optimization as a weighted binary classification problem. Intuitively, we seek an update $\Delta \theta$ that aligns the policy change $\varphi_t^\top \Delta \theta$ with the sign of the advantage $\widehat{A}_t$. By defining a binary label $y_t = \text{sign}(\widehat{A}_t)$ and a sample weight $\eta_t = |\widehat{A}_t|$, we maximize the margin $y_t \varphi_t^\top \Delta \theta$. This is effectively approximated by minimizing the weighted logistic loss, which provides a bounded surrogate for the sign-agreement objective:
\[ v(S_i) := \min_{\Delta \theta} L_{S_j}(\theta^\ast + \Delta \theta) \qquad \text{where } L_{S_j}(\theta^\ast + \Delta \theta) =\frac{1}{|\mathcal{D}_{S_i}|} \sum_{t \in \mathcal{D}_{S_i}} \eta_t \cdot \log \left( 1 + \exp \left( - y_t \tilde{\varphi}_t^\top \Delta \theta \right) \right). \]
This formulation enables rapid estimation of the policy performance for any subset $S_i$ without full retraining. As shown in Table~\ref{table_motivation_result}, the linearization error remains low ($<2\%$) within the trust region.
Another challenge is the high dimensionality of gradient vectors. We address this by applying random Gaussian projections from $\varphi_t$ to a low-dimensional space. By the Johnson–Lindenstrauss Lemma, this projection provably preserves the pairwise similarity between pairs of gradients. Solving the regression problem in the reduced space then takes only a few seconds. %

\begin{table}[h!]
\centering
\caption{We empirically find that the approximation error of $\epsilon$ is negligible for several interactive RL environments. We report the mean and standard deviation of the approximation error across $10$ random subsets. We defer a detailed description of the setup to the experiments section.}\label{table_motivation_result}\vspace{0.075in}
{\small\begin{tabular}{@{}lcccc@{}}
    \toprule
    {Distance} & {MT10} & {CartPole} & {Highway} & {Lunarlander}\\
    \toprule
    $0.1\%$ & $0.01_{\pm0.01}\%$ & $0.12_{\pm0.14}\%$ & $0.02_{\pm0.02}\%$ & $0.06_{\pm0.01}\%$  \\
    $0.5\%$ & $0.43_{\pm0.73}\%$ & $0.73_{\pm0.10}\%$ & $0.11_{\pm0.09}\%$ & $0.03_{\pm0.02}\%$\\
    $1.0\%$ & $0.32_{\pm0.56}\%$ & $0.98_{\pm0.65}\%$ & $2.04_{\pm0.58}\%$  & $0.48_{\pm0.01}\%$\\
    \bottomrule
\end{tabular}}
\end{table}

Specifically, we compute the relative residual sum of squares (RSS) error across several RL environments: MT10 from MetaWorld~\cite{yu2020metaworld}, CartPole, LunarLander, and Highway from Gymnasium~\cite{towers2024gymnasium}. %
In each case, a four-layer MLP policy is trained, and gradients are estimated from 2,048 steps per task.
The results are averaged over $10$ randomly sampled subsets of size $5$. As shown in Table~\ref{table_motivation_result}, the RSS error remains below $2\%$ when $\norm{\theta - \theta^\ast}$ is small with a proper initialization $\theta^\star$, indicating high accuracy in the linear approximation. The error gradually increases to around $10\%$ for perturbations up to $5\%$. These results validate the surrogate model's ability to approximate outcomes near the meta-policy.

\paragraph{Random projection.} We now describe the implementation of the random projection step in our attribution procedure. We sample a random projection matrix $P\in\mathbb{R}^{p\times d}$, where $d$ is smaller than $p$ and every entry of $P$ is independently drawn from a Gaussian distribution $\mathcal{N}(0, d^{-1})$.
We project the high-dimensional gradients to a $d$-dimensional space as \[ \tilde{\varphi}_{i,t}=P^\top \varphi_{i,t}. \]
Then, we solve a logistic regression problem in the $d$-dimensional space with $\theta_d \in \mathbb{R}^d$:
\begin{align}
\widehat{\theta}_d \leftarrow \arg\min_{\theta_d \in \mathbb{R}^d} \frac{1}{|\mathcal{D}_{S_j}|}\sum_{(\varphi, y, w) \in \mathcal{D}_{S_j}} \ell(\tilde{\varphi}, y, w; \theta_d),
\end{align}
for any $j = 1, 2, \dots, m$.
Finally, we map the solution from dimension $d$ back to the original dimension $p$ as
\[ \widehat{\theta}^{(j)} = \theta^\star + P\widehat{\theta}_d. \]

\begin{algorithm}[t!]
\caption{Scalable Kernel Surrogate Estimation using Policy Gradient Estimation (\acronym{})}\label{alg_1}
\textbf{Input:} $n$ subsets $S_1,\dots,S_n$ of $\set{1,\dots,m}$; Initial meta-policy $\pi_{\theta^\star}$\\
\textbf{Require:} Number of episodes $T$;
Projection dimension $d$
\begin{algorithmic}[1]
    \State $\pi_{\theta^\star} \leftarrow$
    Set up the meta-policy for all objectives in $\set{1,\dots,m}$
    \State $P \leftarrow$ A $p$ by $d$ isotropic Gaussian
    random projection matrix with mean zero and variance $\frac 1 d$
    \For{$t=1,\dots,T$}
        \For{$i\in\set{1,\dots,m}$}
            \State $\widetilde{\varphi}_{i,t}
            \leftarrow P^\top
            \left.\nabla_\theta
            \log\pi_\theta(a_{i,t}\mid s_{i,t})
            \right|_{\theta=\theta^\star}$
            \State $\eta_{i,t}\leftarrow|\widehat{A}_{i,t}|$
            \State $y_{i,t}\leftarrow
            \operatorname{sign}(\widehat{A}_{i,t})$
        \EndFor
    \EndFor
    \For{$j=1,\dots,n$}
        \State $\Delta\widehat{\theta}_d^{(j)}
        \leftarrow
        \arg\min_{\Delta\theta_d\in\real^d}
        \widehat{L}_{S_j}
        (\theta^\star+P\Delta\theta_d)$
        \State $\widehat{\theta}^{(j)}
        \leftarrow\theta^\star
        +P\Delta\widehat{\theta}_d^{(j)}$
        \State $\widehat{v}(S_j)
        \leftarrow
        -\widehat{L}_{S_j}(\widehat{\theta}^{(j)})$
    \EndFor
    \State \textbf{return}
    $\widehat{v}(S_1),\dots,\widehat{v}(S_n)$
\end{algorithmic}
\end{algorithm}

Based on the Johnson-Lindenstrauss lemma, when $d = O\Big(\frac{\log p}{\epsilon^2}\Big)$, for any $\varphi_i$ with $\bignorm{\varphi_i} \le G$ and any $\theta, \theta_0$ with $\ell_2$-norm at most $D$, we have
\begin{align*}
    \abs{\inner{g_i}{\theta - \theta_0} - \inner{P g_i}{P (\theta  - \theta_0)}} 
    \le \epsilon \bigabs{\inner{\varphi_i}{\theta - \theta_0}}
    \le {2 G D }{\epsilon}.
\end{align*}
This ensures that the random projection step will preserve the accuracy of the gradient estimation procedure.
See Algorithm \ref{alg_1} for the complete description.

\paragraph{Correlation between attribution scores and LOO analysis.} The goal of this experiment is to evaluate whether an attribution method can predict which trajectory subsets are more useful for policy improvement.
For each task, we first collect a fixed pool of trajectories and sample multiple trajectory subsets from that pool. Each subset is used under the same update protocol to estimate its true downstream effect on the policy. The resulting subset performance can be measured by hypervolume.

Given these true subset effects, each attribution method assigns a predicted score to every subset. We then compare the ranking induced by the predicted scores with that induced by the true subset effects. The LDS score is the Spearman rank correlation between these two rankings. A higher LDS means the attribution method better identifies which trajectory subsets are actually more beneficial.

In Figure \ref{fig_attr_loo}, we illustrate the correlation between the kernel surrogate attribution scores and leave-one-out analysis results.
The results show that these two are strongly correlated across different tests.

\begin{figure}[t]
    \centering
    \begin{subfigure}{0.48\linewidth}
        \centering
        \includegraphics[width=0.60\linewidth]{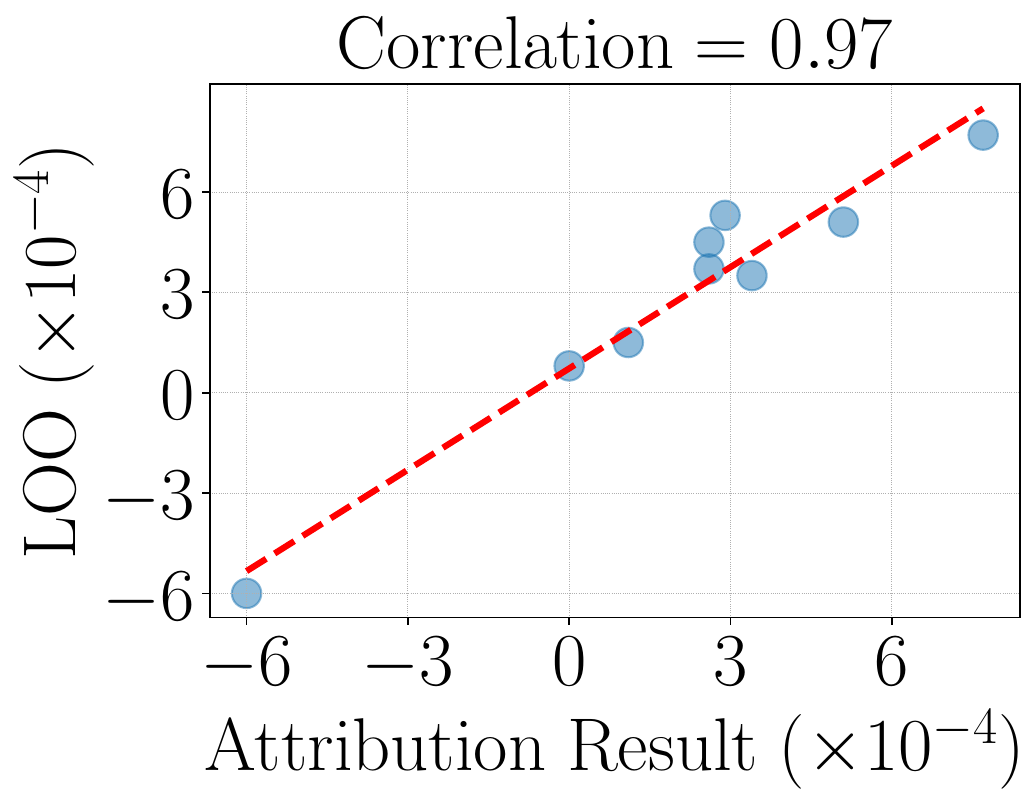}
        \caption{Trajectory-level attribution.}
    \end{subfigure}
    \begin{subfigure}{0.48\linewidth}
        \centering
        \includegraphics[width=0.60\linewidth]{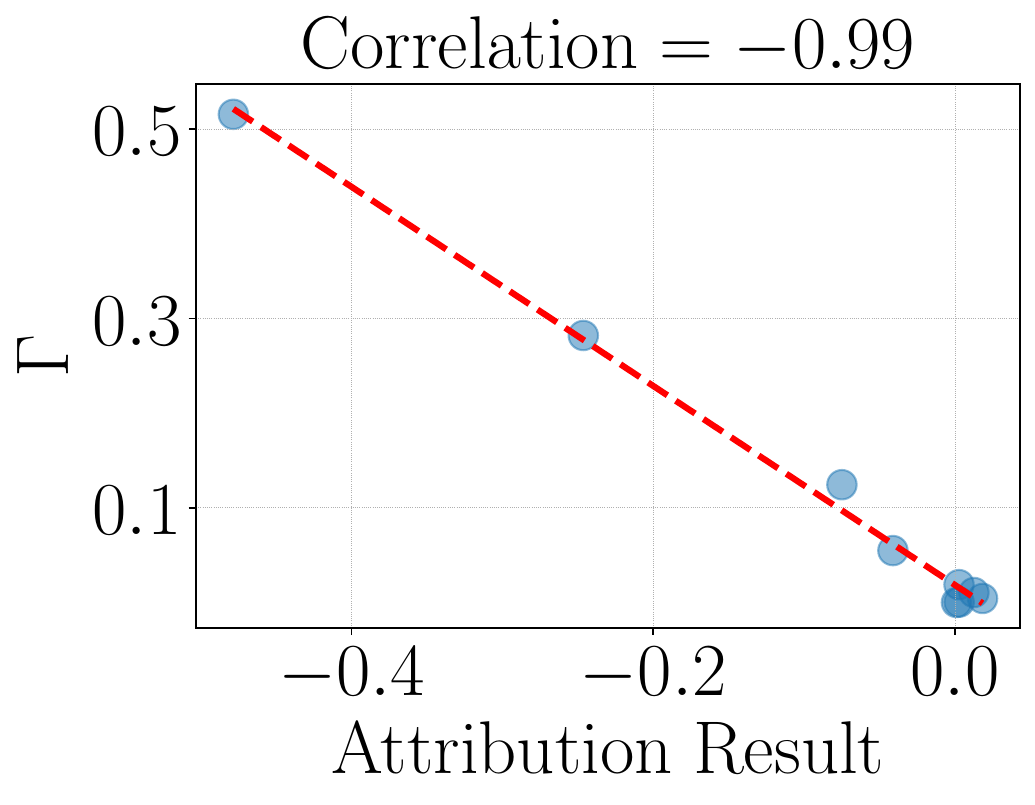}
        \caption{Objective-level attribution.}        
    \end{subfigure}
    \vspace{-0.00in}
    \caption{Correlation analysis of attribution scores against LOO scores (left) and objective conflict coefficient $\Gamma$ (right).}\label{fig_attr_loo}
\end{figure}

\begin{remark}
We describe several natural extensions of our approach.
First, our algorithms naturally extend to online multitask RL \cite{zhang2025scalable}, where the state-action space varies across tasks. %
We leave further exploration of these benchmarks to future work.

Second, kernel surrogates can also be used to measure the influence of an individual trajectory on the overall model behavior.
This can be accomplished by sampling subsets of trajectories instead of objectives in the above algorithm.

Lastly, while we primarily focus on probabilistic policies, similar ideas can be extended to deterministic policies.
\end{remark}

\begin{remark}
Our problem formulation relates to multi-objective optimization in prior literature \cite{papadimitriou2000approximability,chen2024mosh}, and we extend this line of work to the online learning setting. Further exploring the connection between ensemble learning and online boosting \cite{beygelzimer2015optimal} is another direction for future work.

Our work connects to the literature on influence estimation and mechanistic interpretability, which aims to trace model behavior back to influential training data. Prior work include influence functions~\cite{koh2017understanding}, TracIn~\cite{pruthi2020estimating} and linear surrogate models~\cite{liidentification}. In reinforcement learning, recent efforts have adapted these techniques to explain policy decisions, including trajectory-level attribution in offline RL~\cite{deshmukh2023explaining} and snapshot influence functions for filtering harmful data in online RL~\cite{hu2025snapshot}. One of the open questions in this area is to develop better theoretical characterizations of attribution methods.
Our kernel-based objective attribution thus offers a principled approach and accurately estimates second-order interactions across preferences.
This connection seems particularly relevant in light of recent advances in RL reasoning \cite{bateni2025algorithmic}.
\end{remark}

\section{Complete Proofs}\label{sec_proof}

In this section, we present the proof of Theorem \ref{theorem:regret} in Appendix \ref{app:proof}.
We introduce a lower bound in Appendix \ref{app:lower_bound} that matches the upper bound.
Next, we analyze the magnitude of $\Gamma$ in Algorithm \ref{subsec:algorithm-3-analysis}.
Finally, in Appendix \ref{app:approx-realizability}, we present a relaxation when only an approximate low-rank basis assumption holds.

\subsection{Proof of Theorem \ref{theorem:regret}}\label{app:proof}

This section gives a complete proof of Theorem~\ref{theorem:regret}. The proof has four key steps.
Lemma~\ref{lemma:routed-optimism} reduces cumulative regret to a sum of on-policy Bellman residuals.
$\bar d_h$ splits each residual into a \emph{conflict factor} (the $\chi^2$ deviation of the deployed occupancy from $\bar d_h$) and an \emph{in-sample factor} (the squared Bellman residual measured under $\bar d_h$).
Lemma~\ref{lemma:obj-conflict} uses $\Gamma$, and Lemma~\ref{lemma:weight-regularity} uses the schedule regularity in Assumption~\ref{assm:mixture_stability}.
Lemma~\ref{lemma:in-sample-witness} converts the witness-side error into the confidence radius $\beta$. Steps 3 and 4 give a $\sqrt{T}$ stable-phase rate, while the burn-in phase contributes $\mathcal{O}(H t_0)$ with $t_0=\mathcal{O}(C_\cov K\beta)$.

\paragraph{Notations.}
We analyze the {pooled-confidence} version of the routed optimistic algorithm. At the beginning of episode $t$, each basis class $\cF_k$ has a confidence set $\cF_k^{(t-1)}(\beta)$ constructed from the pooled history $\cD^{(t-1)}=\{(\tau^{(i)},w_i):i<t\}$, meaning that the regression residual is computed against \emph{every} past trajectory, not only the trajectories routed to basis $k$. Routing determines which optimistic basis is executed; the confidence guarantee used below controls the Bellman residual of the selected function relative to the cumulative occupancy across all previous episodes. After observing $w_t$, the learner chooses
\begin{equation}\label{eq:routing-rule}
	(k_t,\,f^{(t)}) \in \arg\max_{k\in[K],\,f\in\cF_k^{(t-1)}(\beta)} V_1^f(s_0;w_t),~ V_h^f(s;w):=\max_{a} f_h(s,a;w).
\end{equation}
The deployed policy $\pi^{(t)}$ is greedy with respect to $f^{(t)}$ under $w_t$, with state-action occupancy $d_h^{(t)}:=d_h^{\pi^{(t)}}$. Write the scalarized Bellman residual of the selected optimistic function as
\[
\delta_h^{(t)}(s,a) = f_h^{(t)}(s,a;w_t) - (\cT_h^{w_t} f_{h+1}^{(t)})(s,a),
\]
with the convention $f_{H+1}^{(t)}\equiv 0$. Finally, define the basis-mixture witness $\bar d_h := K^{-1}\sum_{k=1}^K d_h^{\pi^\star_{\basis k}}$ and the un-normalized cumulative empirical occupancy $\bar d_h^{(t)} := \sum_{i=1}^{t-1} d_h^{(i)}$.

Lastly, recall that the $\chi^2$-divergence between two distributions $P(x)$ and $Q(x)$ is defined as \[ \chi^2(P\|Q)=\sum_{x:Q(x)>0}\frac{(P(x)-Q(x))^2}{Q(x)}. \]

\paragraph{Proof structure.}
The proof relies on two events holding simultaneously, each with probability at least $1-\delta/2$ (so that their intersection holds with probability at least $1-\delta$).
For ${E_{\mathrm{opt}}}$, every realizable $Q^{\pi^\star_{\basis k}}(\cdot;w)$ remains in its confidence set $\cF_k^{(t-1)}(\beta)$ for all relevant $t,k,w$. This is the standard optimism event for confidence-set value iteration, union-bounded over the $K$ basis classes.
As for ${E_{\mathrm{reg}}}$, for every selected $f^{(t)}$, the cumulative squared Bellman residual on past trajectories is bounded by $\beta$. This is the standard in-sample regression event, again, union-bounded over the $K$ basis classes.
Next, we work on the intersection of $E_{\mathrm{opt}}$ and $E_{\mathrm{reg}}$.

\begin{lemma}\label{lemma:routed-optimism}
	On the event $E_{\mathrm{opt}}$, for every episode $t$, \[ V_1^{f^{(t)}}(s_0;w_t) \ge \mathcal{J}(\pi^\star_{w_t};w_t). \]
\end{lemma}

\begin{proof}
Assumption~\ref{assm:ref_completeness} expresses the occupancy of $\pi^\star_w$ as a convex combination of the basis occupancies. Since the scalarized value is linear in the occupancy measure,
\[ \mathcal{J}(\pi^\star_w; w) =\sum_k \nu_w(k)\mathcal{J}(\pi^\star_{\basis k};w). \]
By Lemma~\ref{lem_routing}, there exists $k^\star(w_t)\in[K]$ such that the optimal MORL policy and the optimal basis policy at index $k^\star(w_t)$ have the same scalarized return: $\mathcal{J}(\pi^\star_{w_t};w_t)=\mathcal{J}(\pi^\star_{\basis k^\star(w_t)};w_t)$. 

By event $E_{\mathrm{opt}}$, the corresponding true Q-function $Q^{\pi^\star_{\basis k^\star(w_t)}}(\cdot;w_t)$ lies in $\cF_{k^\star(w_t)}^{(t-1)}(\beta)$. Restricting the joint argmax in \eqref{eq:routing-rule} to $k=k^\star(w_t)$ then yields
\begin{align*}
		V_1^{f^{(t)}}(s_0;w_t)
		\;&\ge\; \max_{f\in\cF_{k^\star(w_t)}^{(t-1)}(\beta)}V_1^f(s_0;w_t)
		\;\ge\; V_1^{Q^{\pi^\star_{\basis k^\star(w_t)}}}(s_0;w_t) \\
		\;&=\; \mathcal{J}(\pi^\star_{\basis k^\star(w_t)};w_t)
		\;=\; \mathcal{J}(\pi^\star_{w_t};w_t). \qedhere
	\end{align*}
\end{proof}

Lemma~\ref{lemma:routed-optimism} allows us to upper bound regret by the gap between the optimistic value and the value actually obtained by the greedy policy. The usual telescoping identity for finite-horizon Bellman residuals gives
\begin{align}
    \mathsf{Reg}(T)
    &\le \sum_{t=1}^T \left[V_1^{f^{(t)}}(s_0;w_t)-\mathcal{J}(\pi^{(t)};w_t)\right] \notag\\
    &=  \sum_{t=1}^T\sum_{h=1}^H \mathbb{E}_{(s,a)\sim d_h^{(t)}}\!\left[\delta_h^{(t)}(s,a)\right].
    \label{eq:opt-decomp}
\end{align}
The last equality follows from the standard layerwise telescoping along the trajectory of $\pi^{(t)}$. Because $\pi^{(t)}$ is greedy with respect to $f^{(t)}$, at each step $h$, $V_h^{f^{(t)}}(s_h;w_t)=f_h^{(t)}(s_h,a_h;w_t)$. Subtracting the scalarized reward plus the next optimistic value at each layer telescopes from $V_1^{f^{(t)}}(s_0;w_t)$ down to the realized scalarized return, leaving exactly the expected sum of $\delta_h^{(t)}$.

Let us analyze the term $\mathbb{E}_{(s,a)\sim d_h^{(t)}}\left[\delta_h^{(t)}(s,a)\right]$ in the regret upper bound in equation \eqref{eq:opt-decomp}. For each $t$ and $h$, multiply and divide by $\sqrt{\bar d_h(s,a)}$ inside the inner product:
\begin{align}
	\mathbb{E}_{d_h^{(t)}}[\delta_h^{(t)}]	= \sum_{s,a}\frac{d_h^{(t)}(s,a)}{\sqrt{\bar d_h(s,a)}}\cdot \sqrt{\bar d_h(s,a)}\delta_h^{(t)}(s,a) 
	\le {\sqrt{\phi_h(d_h^{(t)})}} \cdot {\sqrt{\mathbb{E}_{\bar d_h} \big[(\delta_h^{(t)})^2\big]}}.
	\label{eq:cs-witness}
\end{align}

Notably the conflicting factor $\phi_h(d)$ is connected with the $\chi^2$-divergence:
\[\phi_h(d) := \sum_{s,a}\frac{d(s,a)^2}{\bar d_h(s,a)} = 1+\chi^2(d\|\bar d_h). \]

The two factors in equation \eqref{eq:cs-witness} are controlled separately: the conflict factor by Assumption~\ref{assm:mixture_stability} and the in-sample factor by combining the witness-comparison bound with the pooled-regression guarantee.

\begin{lemma}\label{lemma:obj-conflict}
Under Assumption~\ref{assm:mixture_stability}, for every $t$ and $h$,
\[ \phi_h(d_h^{(t)}) := \sum_{s,a}\frac{d_h^{(t)}(s,a)^2}{\bar d_h(s,a)} =1+\chi^2(d_h^{(t)}\|\bar d_h) \le 1+\Gamma. \]
\end{lemma}
\begin{proof}
By Assumption~\ref{assm:mixture_stability}, $d_h^{(t)}=\sum_k\nu^{(t)}(k)d_h^{\pi^\star_{\basis k}}$. The map
\[ d\mapsto \phi_h(d)=\sum_{s,a}\frac{d(s,a)^2}{\bar d_h(s,a)} \]
is a squared norm in the measure weighted by $\bar d_h^{-1}$, hence convex. Jensen's inequality gives
\[ \phi_h(d_h^{(t)})
    \le \sum_k \nu^{(t)}(k)\phi_h(d_h^{\pi^\star_{\basis k}})
    \le 1+\Gamma, \]
where the last step is Definition~\ref{def:obj_conflict_coef}.
\end{proof}

\begin{lemma}\label{lemma:weight-regularity}
	Under Assumption~\ref{assm:mixture_stability} (with constants $c_0\in(0,1]$ and burn-in time $t_0\in\mathbb{N}$), for every $t\ge t_0$, $h\in[H]$, and $(s,a)$,
	\[	\bar d_h^{(t)}(s,a) \;\ge\; \frac{c_0\,(t-1)}{C_\cov}\,\bar d_h(s,a).	\]
\end{lemma}
\begin{proof}
The cumulative empirical occupancy can be expanded in terms of the basis occupancies:
\[ \frac{\bar d_h^{(t)}(s,a)}{t-1}
    =\sum_{k=1}^K \overline\nu^{(t)}(k)d_h^{\pi^\star_{\basis k}}(s,a)
    \ge \frac{c_0}{C_\cov K}\sum_{k=1}^K d_h^{\pi^\star_{\basis k}}(s,a)
    =\frac{c_0}{C_\cov}\bar d_h(s,a). \]
The inequality is exactly the stable-phase lower bound in Assumption~\ref{assm:mixture_stability}. Thus after burn-in, the data distribution accumulated by the algorithm pointwise dominates the uniform basis mixture up to the factor $C_\cov/c_0$.
\end{proof}

\begin{lemma}\label{lemma:in-sample-witness}
On event $E_{\mathrm{reg}}$, for every $t\ge t_0$ and $h\in[H]$,
\[ \mathbb{E}_{\bar d_h}\!\left[(\delta_h^{(t)})^2\right] \le \frac{C_\cov\,\beta}{c_0(t-1)}. \]
\end{lemma}
\begin{proof}
The standard regression guarantee for optimistic confidence sets, union bounded over $k\in[K]$, gives for the selected $f^{(t)}$ (i.e., on the event $E_{\mathrm{reg}}$),
\begin{equation}\label{eq:in-sample-empirical}
    \sum_{i=1}^{t-1}\mathbb{E}_{d_h^{(i)}}\!\left[(\delta_h^{(t)})^2\right]
    =\mathbb{E}_{\bar d_h^{(t)}}\!\left[(\delta_h^{(t)})^2\right]
    \le \beta .
\end{equation}
This is exactly where the proof uses pooled confidence updates: the measure on the left is the cumulative occupancy across all previous episodes, not just those routed to $k_t$. Lemma~\ref{lemma:weight-regularity} gives
\[ \bar d_h(s,a)\le \frac{C_\cov}{c_0(t-1)}\bar d_h^{(t)}(s,a), \text{ for } t\ge t_0. \]
Multiplying this pointwise inequality by $(\delta_h^{(t)}(s,a))^2$ and summing over $(s,a)$, then applying \eqref{eq:in-sample-empirical}, yields
\[ \mathbb{E}_{\bar d_h}\!\left[(\delta_h^{(t)})^2\right] \le \frac{C_\cov}{c_0(t-1)} \mathbb{E}_{\bar d_h^{(t)}}\!\left[(\delta_h^{(t)})^2\right] \le \frac{C_\cov\beta}{c_0(t-1)}. \]
\end{proof}

\begin{proof}[Proof of Theorem~\ref{theorem:regret}]
	Combining equation \eqref{eq:cs-witness} with Lemmas~\ref{lemma:obj-conflict} and~\ref{lemma:in-sample-witness} gives, for every $t\ge t_0$ and $h\in[H]$,
	\begin{equation}\label{eq:per-stage-regret}
		\mathbb{E}_{d_h^{(t)}}[\delta_h^{(t)}] \le \sqrt{1+\Gamma}\cdot \sqrt{\frac{C_\cov \beta}{c_0 (t-1)}} = \sqrt{\frac{C_\cov\,(1+\Gamma)\beta}{c_0 (t-1)}}.
	\end{equation}
	We split the regret \eqref{eq:opt-decomp} at the burn-in time $t_0$:
	\begin{equation*}
		\mathsf{Reg}(T) \;\le\; \underbrace{\sum_{t=1}^{t_0-1}\sum_{h=1}^H\mathbb{E}_{d_h^{(t)}}[\delta_h^{(t)}]}_{\text{(I)}} \;+\; \underbrace{\sum_{t=t_0}^T\sum_{h=1}^H\mathbb{E}_{d_h^{(t)}}[\delta_h^{(t)}]}_{\text{(II)}}.
	\end{equation*}
	
	During burn-in, for each $t<t_0$, $\sum_h \mathbb{E}_{d_h^{(t)}}[\delta_h^{(t)}] = V_1^{f^{(t)}}(s_0;w_t) - \mathcal{J}(\pi^{(t)};w_t) \le H$ since both $V_1^{f^{(t)}}\in[0,H]$ and $\mathcal{J}(\pi^{(t)};w_t)\in[0,H]$. Hence $(\text{I})\le H\,t_0$.
	
	During the stable phase, substituting \eqref{eq:per-stage-regret} and summing over $h$ and $t$, \text{II} is at most
	\begin{align*}
	H\sum_{t=t_0}^T \sqrt{\frac{C_\cov\,(1+\Gamma)\,\beta}{c_0\,(t-1)}}= H\sqrt{\frac{C_\cov\,(1+\Gamma)\,\beta}{c_0}}
	\sum_{t=t_0}^T\frac{1}{\sqrt{t-1}}
	\le 2H\sqrt{\frac{C_\cov\,(1+\Gamma)\,\beta\,T}{c_0}},
	\end{align*}
	using $\sum_{t=t_0}^T (t-1)^{-1/2}\le \sum_{u=1}^{T-1}u^{-1/2}\le 2\sqrt{T}$.
	
	Adding (I) and (II) and absorbing the universal constant $1/c_0$ into $\mathcal{O}(\cdot)$,
	\[
	\mathsf{Reg}(T) \le H t_0 + 2H\sqrt{\frac{C_\cov(1+\Gamma)\beta T}{c_0}} = \mathcal{O}\!\left(H\sqrt{C_\cov (1+\Gamma) \beta T}\right) + \mathcal{O}(H t_0).
	\]
	The preference process satisfies the schedule regularity of Assumption~\ref{assm:mixture_stability} after $t_0=\mathcal{O}(C_\cov K \beta)$ episodes (a standard Hoeffding's bound on the visitation counts of the $K$ basis components, with the $\log$ factor absorbed into $\beta$). Substituting yields
	\[ \mathsf{Reg}(T) \le \mathcal{O}\!\left(H\sqrt{C_\cov\,(1+\Gamma)\beta T} + H C_\cov K \beta\right), \]
	which is equation \eqref{eq:regret_explicit}. The corresponding worst-case scaling follows by applying $1+\Gamma\le K$ from Lemma~\ref{lem:gamma_K_ub}.
\end{proof}

For completeness, we record the structural lemma used above; it shows that the conflict coefficient is at most $K-1$ for any basis, with equality (up to lower-order terms) in the disjoint regime that powers the lower bound.

\begin{lemma}\label{lem:gamma_K_ub}
For any collection of $K$ basis occupancies $\{d_h^{\pi^\star_{\basis k}}\}_{k\in[K]}$ and every $h\in[H]$,
\[ \Gamma = \max_{h\in[H]}\max_{k\in[K]} \chi^2\!\left(d_h^{\pi^\star_{\basis k}} \big\| \bar d_h\right) \le K-1. \]
Equality holds (up to lower-order terms) when the basis occupancies have pairwise disjoint supports.
\end{lemma}
\begin{proof}
    Fix $h$ and $k$. Pointwise, $\bar d_h(s,a) \ge \frac{1}{K}\, d_h^{\pi^\star_{\basis k}}(s,a)$ on the support of $d_h^{\pi^\star_{\basis k}}$. Hence
    \[ \chi^2\!\left(d_h^{\pi^\star_{\basis k}}\,\big\|\,\bar d_h\right) = \sum_{s,a}\frac{(d_h^{\pi^\star_{\basis k}}(s,a))^2}{\bar d_h(s,a)}-1 \le \sum_{s,a}\frac{(d_h^{\pi^\star_{\basis k}}(s,a))^2}{(1/K)\,d_h^{\pi^\star_{\basis k}}(s,a)}-1 = K-1. \]
    The disjoint-support case attains equality because then $\bar d_h(s,a) = \frac{1}{K}\, d_h^{\pi^\star_{\basis k}}(s,a)$ on the support.
\end{proof}

\subsection{Proof of the regret lower bound}\label{app:lower_bound}

\begin{theorem}\label{theorem:lower_bound}
    There exists a family of MOMDPs with $K$ basis policies, $C_\cov\le K$, and $\Gamma=K-1$, such that for any (possibly randomized) learning algorithm $A$ and any $T\ge K$, a preference sequence that cycles uniformly through $\{e_1,\ldots,e_K\}$ satisfies
    \[ \mathbb{E}\bigl[\mathsf{Reg}(T)\bigr] \ge \Omega\!\left( H\sqrt{K T} \right) = \Omega\!\left( H\sqrt{(1+\Gamma) T} \right). \]
\end{theorem}
Together, Theorems~\ref{theorem:regret} and~\ref{theorem:lower_bound} pin down the dependence on $K$ and $\Gamma$ up to factors of $\sqrt{C_\cov\,\log|\cF|}$.

In the following proof, we prove the lower bound by embedding a standard Bernoulli two-armed bandit problem into a simple subclass of MOMDPs. This is sufficient because a lower bound for one valid MOMDP family is automatically a lower bound for the larger MOMDP class. The constructed family has $K$ pairwise-disjoint sub-chains, one per objective, each containing one Bernoulli two-armed bandit. Here, a sub-chain means a branch of the environment state space, not a policy: the $k$-th optimal basis policy is the policy that routes into sub-chain $\cS_k$ and then plays the better binary action.

Under a cyclic preference sequence, regret has two sources: 
1) a {routing} term for episodes whose first action does not match the active preference;
2) $K$ independent {bandit} terms for wrong binary choices inside correctly routed sub-chains.

We then show that for each preference $e_k$, the learner faces a two-armed bandit on the episodes with active preference $e_k$, with mis-routing acting as a no-information mistake that also costs order $\Delta$. The horizon does not create a harder statistical problem in this construction: after the binary decision, the remaining states reveal no additional choices or independent samples. Instead, the same Bernoulli outcome is paid over the remaining $H-1$ stages, so the bandit regret is simply multiplied by $H-1$. Tuning the bandit gap to $\Delta = \tfrac{1}{8}\sqrt{K/T}$ makes the total unavoidable cost scale as $\Delta T=\Theta(\sqrt{KT})$, and the repeated payoff yields the claimed $\Omega(H\sqrt{KT})$ lower bound.

\paragraph{Hard MOMDP instances.}
Fix $K\ge 2$, $H\ge 2$, and $\Delta\in(0,1/4]$. The MOMDP has $m=K$ objectives and a single start state $s_0$. From $s_0$, the $K$ deterministic actions $a\in[K]$ route the agent to disjoint state-space branches $\cS_1,\ldots,\cS_K$, which we call sub-chains. Inside sub-chain $\cS_k$, the agent makes one binary decision $u\in\{0,1\}$ at the second step; thereafter, the chain is uninformative. There is an unknown bit $b_k\in\{0,1\}$, and the $k$-th reward coordinate at every step inside $\cS_k$ equals a single Bernoulli outcome drawn at the moment of the binary decision, with mean
\[ \tfrac{1}{2}+\Delta\quad\text{if }u=b_k, \tfrac{1}{2}\quad\text{if }u\neq b_k. \]
All other reward coordinates are zero. Thus, each episode yields only one statistically informative bandit observation, even though it has a horizon of $H$. A correctly routed episode where the wrong binary action is chosen contributes regret $(H-1)\Delta$, and an incorrectly routed episode contributes regret at least $(H-1)\Delta$ as well.

\paragraph{Preference sequences.}
The preference sequence cycles through $e_1,e_2,\ldots,e_K,e_1,e_2,\ldots$, so each preference $e_k$ appears between $\lfloor T/K\rfloor$ and $\lceil T/K\rceil$ times. The optimal basis policy $\pi^\star_{\basis k}$ routes to sub-chain $\cS_k$ and plays $b_k$; the $K$ basis occupancies are supported on disjoint sub-chains, so by Lemma~\ref{lem:gamma_K_ub}, $\Gamma=K-1$, and the basis-mixture witness $\bar d_h$ gives $C_\cov\le K$.
We use a lower bound result from the work of \citet{lattimore2020bandit}.

\begin{lemma}[Bandit lower bound \cite{lattimore2020bandit}]\label{lem:bandit_lb}
	For any randomized algorithm interacting with a Bernoulli two-armed bandit of gap $\Delta\in(0,1/4]$, the expected cumulative regret after $n$ pulls is at least $c_{\mathrm{LB}}\min(\Delta n,\,1/\Delta)$ for a universal constant $c_{\mathrm{LB}}>0$, in the worst case over the unknown optimal arm.
\end{lemma}

\begin{proof}[Proof of Theorem~\ref{theorem:lower_bound}]
	We proceed in three steps.
	First, for a fixed interaction history, partition the $T$ episodes into two disjoint sets:
    $M$ episodes where the learner \emph{mis-routes}, meaning that its first action does not equal the index $k_t$ of the active preference $e_{k_t}$. Each such episode loses at least $(H-1)\Delta$ relative to the optimal policy for $e_{k_t}$.
	For each $k\in[K]$, $n_k$ episodes where the learner correctly routes to $\cS_k$ under preference $e_k$. Conditional on these episodes, the learner observes exactly $n_k$ pulls from the two-armed Bernoulli bandit associated with unknown bit $b_k$.

	Let $M_k$ be the number of mis-routed episodes whose active preference is $e_k$, and let $m_k:=M_k+n_k$ be the total number of appearances of preference $e_k$. The cyclic sequence fixes $m_k$ deterministically, with $\lfloor T/K\rfloor\le m_k\le\lceil T/K\rceil$, and $M+\sum_{k=1}^K n_k = \sum_{k=1}^K m_k=T$ pathwise. Taking expectation over the learner's randomness and the bandit rewards, the cumulative regret satisfies
    \begin{equation}\label{eq:lb-decomp}
        \mathbb{E}[\mathsf{Reg}(T)] \ge (H-1)\Delta\,\mathbb{E}[M] + (H-1)\sum_{k=1}^K \mathbb{E}[\mathsf{Reg}_k^{\mathrm{bandit}}(n_k)],
    \end{equation}
    where $\mathsf{Reg}_k^{\mathrm{bandit}}(n_k)$ denotes the cumulative regret of the induced bandit on sub-chain $\cS_k$ (the leading $H-1$ factor accounts for the fact that each Bernoulli outcome is paid out over $H-1$ steps).

    Next, set $\Delta = \tfrac{1}{8}\sqrt{K/T}$, which satisfies $\Delta\le 1/8$ since $T\ge K$. The cyclic preference sequence ensures $m_k \le \lceil T/K\rceil \le 2T/K$, so
    \[ \Delta m_k \le \tfrac{1}{8}\sqrt{K/T}\cdot \tfrac{2T}{K}    = \tfrac{1}{4}\sqrt{T/K}
    \le \tfrac{1}{\Delta}. \]
    For a fixed $k$, restrict attention to the $m_k$ episodes whose active preference is $e_k$. In these episodes, the learner must either route to $\cS_k$ and choose one of the two bandit arms, or mis-route and incur one-stage scalarized regret at least $\Delta$ while receiving no observation from the $k$-th bandit. Thus, the sum $\Delta M_k+\mathsf{Reg}_k^{\mathrm{bandit}}(n_k)$ is lower bounded by the regret of a Bernoulli two-armed bandit problem with $m_k$ opportunities. Lemma~\ref{lem:bandit_lb} gives
    \[\mathbb{E}\!\left[\Delta M_k+\mathsf{Reg}_k^{\mathrm{bandit}}(n_k)\right] \ge c_{\mathrm{LB}}\min(\Delta m_k,1/\Delta)=c_{\mathrm{LB}}\Delta m_k.\]
    Summing over $k$ yields
    \[ \Delta \mathbb{E}[M] + \sum_{k=1}^K \mathbb{E}[\mathsf{Reg}_k^{\mathrm{bandit}}(n_k)] \ge c_{\mathrm{LB}}\Delta\sum_{k=1}^K m_k = c_{\mathrm{LB}}\Delta T. \]
    Lastly, substituting the preceding bound into equation \eqref{eq:lb-decomp},
    \begin{align*}
        \mathbb{E}[\mathsf{Reg}(T)] \ge c_{\mathrm{LB}} (H-1)\Delta T,
    \end{align*}
    Substituting $\Delta=\tfrac{1}{8}\sqrt{K/T}$, we have
    \[ \mathbb{E}[\mathsf{Reg}(T)] \ge \frac{c_{\mathrm{LB}}}{8}(H-1)\sqrt{KT} =\Omega\!\left(H\sqrt{KT}\right). \]
	Since the construction has $\Gamma=K-1$, this is equivalent to $\Omega(H\sqrt{(1+\Gamma)T})$. Thus, we have completed the proof of the lower bound.
\end{proof}

\subsection{Bounds on objective conflict coefficients}\label{subsec:algorithm-3-analysis}

Although Algorithm~\ref{alg_routing} does not require $\Gamma$ as an input, its off-policy transfer step is governed by the same overlap quantity. At iteration $t$, sampling a basis policy uniformly yields the routed-ensemble mixture and its associated conflict measure
\begin{equation}\label{eq:routed-ensemble-mixture}
    \bar d_{h,t}^{\mathrm{RE}} = \frac{1}{K}\sum_{j=1}^K d_h^{\pi_j^{(t-1)}}, ~ \Gamma_{\mathrm{RE}}^{(t)} = \max_{h\in[H]}\max_{j\in[K]} \chi^2\!\left(d_h^{\pi_j^{(t-1)}}\big\|\bar d_{h,t}^{\mathrm{RE}}\right).
\end{equation}
When the learned ensemble recovers the latent basis policies, $\Gamma_{\mathrm{RE}}^{(t)}$ coincides with $\Gamma$. The following proposition formalizes that the off-policy estimators used to route and update the ensemble have variance inflated by exactly the factor $1+\Gamma_{\mathrm{RE}}^{(t)}$.

\begin{proposition}\label{prop:offpolicy-conflict}
Fix an iteration $t$ of Algorithm~\ref{alg_routing}, a stage $h$, and a target basis index $j\in[K]$. Suppose $B$ replay samples $(s_i,a_i)$ are drawn i.i.d.\ from $\bar d_{h,t}^{\mathrm{RE}}$, and let $g_h:\mathcal{S}\times\mathcal{A}\to[-G,G]$ be any bounded measurable signal with target value
\[ \theta_{j,h} = \mathbb{E}_{d_h^{\pi_j^{(t-1)}}}[g_h(s,a)]. \]
The importance-weighted off-policy estimator
\[ \widehat\theta_{j,h} := \frac{1}{B}\sum_{i=1}^B \frac{d_h^{\pi_j^{(t-1)}}(s_i,a_i)}{\bar d_{h,t}^{\mathrm{RE}}(s_i,a_i)}g_h(s_i,a_i) \]
is unbiased and satisfies
\begin{equation}\label{eq:offpolicy-mse}
    \mathbb{E}\!\left[(\widehat\theta_{j,h}-\theta_{j,h})^2\right] \le \frac{G^2\bigl(1+\Gamma_{\mathrm{RE}}^{(t)}\bigr)}{B}.
\end{equation}
In the idealized case $\pi_j^{(t-1)}=\pi^\star_{\basis j}$ for all $j$, $\Gamma_{\mathrm{RE}}^{(t)}=\Gamma$.
\end{proposition}

\begin{proof}
Fix iteration $t$, stage $h$, and target basis index $j$. For brevity, write $d_j := d_h^{\pi_j^{(t-1)}}$, $\mu := \bar d_{h,t}^{\mathrm{RE}}$, and $\rho_j(s,a) := d_j(s,a)/\mu(s,a)$. Finite $\Gamma_{\mathrm{RE}}^{(t)}$ implies $d_j\ll\mu$, so $\rho_j$ is well defined on $\mathrm{supp}(\mu)$. Since $d_j$ and $\mu$ are probability distributions over $\cS\times\cA$ at stage $h$, $\sum_{s,a}d_j = \sum_{s,a}\mu = 1$.

Each replay sample $(s_i,a_i)$ is drawn from $\mu$, so
\[ \mathbb{E}\!\left[\rho_j(s_i,a_i) g_h(s_i,a_i)\right] = \sum_{s,a}\mu(s,a) \frac{d_j(s,a)}{\mu(s,a)} g_h(s,a) = \sum_{s,a}d_j(s,a) g_h(s,a) = \theta_{j,h}. \]
Averaging over $B$ samples preserves the mean: $\mathbb{E}[\widehat\theta_{j,h}] = \theta_{j,h}$.

For a single sample, the second moment of the importance-weighted signal is bounded by
\begin{equation}\label{eq:offpolicy-second-moment}
    \mathbb{E}_\mu\!\left[\rho_j(s,a)^2 g_h(s,a)^2\right] \le G^2\sum_{s,a}\frac{d_j(s,a)^2}{\mu(s,a)} = G^2\bigl(1+\chi^2(d_j\|\mu)\bigr) \le G^2\bigl(1+\Gamma_{\mathrm{RE}}^{(t)}\bigr),
\end{equation}
where the equality uses the standard $\chi^2$-identity $\sum_{s,a} d_j^2/\mu = 1+\chi^2(d_j\|\mu)$ (which holds because $d_j$ and $\mu$ are probability distributions), and the last inequality is the definition of $\Gamma_{\mathrm{RE}}^{(t)}$. The $B$ samples are independent, so
\[ \mathbb{E}\!\left[(\widehat\theta_{j,h} - \theta_{j,h})^2\right] = \mathrm{Var}(\widehat\theta_{j,h}) \le \frac{\mathbb{E}_\mu\!\left[\rho_j^2 g_h^2\right]}{B} \le \frac{G^2\bigl(1+\Gamma_{\mathrm{RE}}^{(t)}\bigr)}{B}, \]
which leads to equation \eqref{eq:offpolicy-mse}.
If $\pi_j^{(t-1)}=\pi^\star_{\basis j}$ for every $j\in[K]$, then $\mu=\bar d_h$ and $\Gamma_{\mathrm{RE}}^{(t)}=\Gamma$ by Definition~\ref{def:obj_conflict_coef}.
\end{proof}

\subsection{Approximate low-rank conditions}\label{app:approx-realizability}

Next, we show how the regret bound changes when Assumption~\ref{assm:ref_completeness} only holds approximately. We separate two effects.
Relative to the best policy in the chosen $K$-policy basis, the proof of Theorem~\ref{theorem:regret} still gives a sublinear regret bound, with the conflict factor inflated by $\varepsilon$.
If the chosen basis does not contain a truly optimal policy for a preference, then true regret also includes the gap between the true optimum and the best basis policy. 

Let $\bar d_h=K^{-1}\sum_{k=1}^K d_h^{\pi^\star_{\basis k}}$ be the same basis-mixture witness used in the proof of Theorem~\ref{theorem:regret}. We replace Assumption~\ref{assm:ref_completeness} by the following approximate version.

\begin{assumption}\label{assm:approx_completeness}
	There exist latent basis policies $\{\pi^\star_{\basis 1},\ldots,\pi^\star_{\basis K}\}$ and a tolerance $\varepsilon\ge0$ such that every deployed policy $\pi\in\{\pi^{(t)}:t\le T\}$ admits a single mixing vector $\nu_\pi\in\Delta^{K-1}$ satisfying, for every stage $h\in[H]$,
	\begin{equation*}%
		d_h^\pi	= q_{\pi,h}+e_{\pi,h}, q_{\pi,h}:=\sum_{k=1}^K\nu_\pi(k)d_h^{\pi^\star_{\basis k}}, \|e_{\pi,h}\|_{\bar d_h^{-1}}^2 := \sum_{s,a}\frac{e_{\pi,h}(s,a)^2}{\bar d_h(s,a)}\le \varepsilon^2 .
	\end{equation*}
    The norm is interpreted with the usual support convention: finite norm requires $e_{\pi,h}(s,a)=0$ whenever $\bar d_h(s,a)=0$.
\end{assumption}

The $\bar d_h^{-1}$-weighted norm is chosen to match the Cauchy--Schwarz step in Theorem~\ref{theorem:regret}: the proof controls $\phi_h(d)=\sum_{s,a}d(s,a)^2/\bar d_h(s,a)=\|d\|_{\bar d_h^{-1}}^2$. Thus, an error bound in this same norm can be carried through the original argument by the triangle inequality.

For this approximate statement, we keep the stable-coverage consequence of Assumption~\ref{assm:mixture_stability}: for all $t\ge t_0$, $h\in[H]$, and $(s,a)$,
\begin{equation}\label{eq:approx-stable-coverage}
	\bar d_h^{(t)}(s,a):=\sum_{i=1}^{t-1}d_h^{(i)}(s,a) \;\ge\; \frac{c_0(t-1)}{C_\cov}\,\bar d_h(s,a).
\end{equation}
This is the only property of coverage needed for the pooled-regression step. It can be assumed directly or derived from an approximate exploration condition that guarantees the accumulated actual occupancies dominate the basis mixture.

For a preference $w$, define the best value achievable by the chosen basis as
\begin{equation*}%
	\mathcal{J}_K^\star(w) := \max_{k\in[K]}\mathcal{J}(\pi^\star_{\basis k};w) = \max_{\nu\in\Delta^{K-1}}\sum_{k=1}^K\nu(k)\mathcal{J}(\pi^\star_{\basis k};w).
\end{equation*}
The equality holds because the objective is linear in the mixing weights. We measure basis-relative regret by
\begin{equation*}%
	\mathsf{Reg}_K(T) := \sum_{t=1}^T \left[\mathcal{J}_K^\star(w_t)-\mathcal{J}(\pi^{(t)};w_t)\right].
\end{equation*}

\begin{proposition}\label{prop:approx-regret}
	Suppose scalarized rewards are bounded in $[0,1]$ per step, Bellman realizability and completeness hold for each per-basis class $\cF_k$, Assumption~\ref{assm:approx_completeness} holds with tolerance $\varepsilon$, and the stable-coverage condition \eqref{eq:approx-stable-coverage} holds after a burn-in time $t_0=\mathcal{O}(C_\cov K\beta)$, where $\beta:=c_1\log(K|\cF|T/\delta)$. Then, with probability at least $1-\delta$, the pooled-confidence routed-ensemble algorithm satisfies
	\begin{equation}\label{eq:approx-regret}
		\mathsf{Reg}_K(T)\;\le\;\mathcal{O}\!\left(H\sqrt{C_\cov\beta T}\,\bigl(\sqrt{1+\Gamma}+\varepsilon\bigr)+HC_\cov K\beta\right).
	\end{equation}
	When $\varepsilon=0$ and the basis is exact, $\mathsf{Reg}_K(T)=\mathsf{Reg}(T)$ and this reduces to Theorem~\ref{theorem:regret}.
\end{proposition}

\begin{proof}%
The proof follows the same sketch as Theorem~\ref{theorem:regret} (Appendix~\ref{app:proof}), with two simplifications and one change. 
First, the comparator is changed from the true optimal value $\mathcal{J}(\pi^\star_w;w)$ to the basis-oracle value $\mathcal{J}_K^\star(w)$, so routed optimism becomes \emph{exact} against this weaker comparator (Lemma~\ref{lemma:basis-routed-optimism} below).
Second, the conflict factor $\sqrt{\phi_h(d_h^{(t)})}$ inflates from $\sqrt{1+\Gamma}$ (Lemma~\ref{lemma:obj-conflict}) to $\sqrt{1+\Gamma}+\varepsilon$ (Lemma~\ref{lemma:approx-obj-conflict} below), because deployed occupancies are only approximately in the basis hull.
Third, the witness-comparison and pooled-regression steps are \emph{unchanged}: they use only the stable coverage condition \eqref{eq:approx-stable-coverage} and Bellman completeness, neither of which is affected by relaxing Assumption~\ref{assm:ref_completeness}. We carry out the modifications below.

First, we show the following claim on the optimism of the routed basis.

\begin{claim}\label{lemma:basis-routed-optimism}
	On the optimism event $E_{\mathrm{opt}}$ (defined in Appendix~\ref{app:proof}), for every episode $t$,
	\[V_1^{f^{(t)}}(s_0;w_t)\;\ge\;\mathcal{J}_K^\star(w_t).\]
\end{claim}
\begin{proof}
	Let $k_K^\star(w_t)\in\arg\max_{k\in[K]}\mathcal{J}(\pi^\star_{\basis k};w_t)$. By the optimism event $E_{\mathrm{opt}}$, the true $Q$-function $Q^{\pi^\star_{\basis k_K^\star(w_t)}}(\cdot;w_t)$ belongs to the confidence set $\cF_{k_K^\star(w_t)}^{(t-1)}(\beta)$. The routed optimistic choice in equation \eqref{eq:routing-rule} maximizes over all basis indices and all functions in their confidence sets, so we may restrict the maximization to the single feasible candidate $(k_K^\star(w_t),Q^{\pi^\star_{\basis k_K^\star(w_t)}})$. This gives
	\[V_1^{f^{(t)}}(s_0;w_t)\;\ge\;V_1^{Q^{\pi^\star_{\basis k_K^\star(w_t)}}}(s_0;w_t)\;=\;\mathcal{J}(\pi^\star_{\basis k_K^\star(w_t)};w_t)\;=\;\mathcal{J}_K^\star(w_t). \qedhere\]
\end{proof}

Second, we demonstrate regret reduction via a certain identity.
This step is identical to the regret reduction in equation \eqref{eq:opt-decomp}. Working on the intersection $E_{\mathrm{opt}}\cap E_{\mathrm{reg}}$ (probability at least $1-\delta$), Lemma~\ref{lemma:basis-routed-optimism} together with the standard layerwise telescoping identity gives
\begin{equation}\label{eq:approx-regret-reduction}
	\mathsf{Reg}_K(T)\;\le\;\sum_{t=1}^T\!\left[V_1^{f^{(t)}}(s_0;w_t)-\mathcal{J}(\pi^{(t)};w_t)\right]\;=\;\sum_{t=1}^T\sum_{h=1}^H \mathbb{E}_{d_h^{(t)}}\!\left[\delta_h^{(t)}\right].
\end{equation}
Below, we control the Bellman residual term on the right-hand side.

Identical to equation \eqref{eq:cs-witness}, for every $t\in[T]$ and $h\in[H]$, we have
\begin{equation}\label{eq:approx-cs}
\mathbb{E}_{d_h^{(t)}} \left[\delta_h^{(t)}\right]\le \sqrt{\phi_h(d_h^{(t)})}  \cdot \sqrt{\mathbb{E}_{\bar d_h} \left[(\delta_h^{(t)})^2\right]}.
\end{equation}

Next, we show the following approximate conflict bound, which is the only step affected by $\varepsilon$.

\begin{claim}\label{lemma:approx-obj-conflict}
	Every policy $\pi$ satisfying the decomposition in Assumption~\ref{assm:approx_completeness} also satisfies, for every $h\in[H]$,
	\begin{equation}\label{eq:approx-phi-bound}
		\sqrt{\phi_h(d_h^\pi)}\;\le\;\sqrt{1+\Gamma}+\varepsilon.
	\end{equation}
\end{claim}

\begin{proof}
	Since $\phi_h(d)=\|d\|_{\bar d_h^{-1}}^2$, Assumption~\ref{assm:approx_completeness} and the triangle inequality in $\ell_2$ norm give
	\[\sqrt{\phi_h(d_h^\pi)}\;=\;\|q_{\pi,h}+e_{\pi,h}\|_{\bar d_h^{-1}}\;\le\;\|q_{\pi,h}\|_{\bar d_h^{-1}}+\|e_{\pi,h}\|_{\bar d_h^{-1}}\;\le\;\sqrt{\phi_h(q_{\pi,h})}+\varepsilon .\]
	The first summand is controlled exactly as in Lemma~\ref{lemma:obj-conflict}: by convexity of $\phi_h$ and Definition~\ref{def:obj_conflict_coef},
	\[
	\phi_h(q_{\pi,h})
	\;=\;
	\phi_h\!\left(\textstyle\sum_k\nu_\pi(k)\,d_h^{\pi^\star_{\basis k}}\right)
	\;\le\;
	\sum_k\nu_\pi(k)\,\phi_h(d_h^{\pi^\star_{\basis k}})
	\;\le\;
	1+\Gamma .
	\]
	Combining this display with the triangle-inequality bound proves equation \eqref{eq:approx-phi-bound}.
    Applying Lemma~\ref{lemma:approx-obj-conflict} to the deployed policy $\pi^{(t)}$ yields
    \begin{equation}\label{eq:approx-conflict-deployed}
	\sqrt{\phi_h(d_h^{(t)})}\;\le\;\sqrt{1+\Gamma}+\varepsilon .
\end{equation}
\end{proof}

The cumulative-occupancy lower bound in equation \eqref{eq:approx-stable-coverage}, combined with the pooled-regression event $E_{\mathrm{reg}}$, gives the same in-sample-to-witness bound as Lemma~\ref{lemma:in-sample-witness}: for every $t\ge t_0$ and $h\in[H]$,
\begin{equation}\label{eq:approx-in-sample}
	\sqrt{\mathbb{E}_{\bar d_h}\!\left[(\delta_h^{(t)})^2\right]}
	\;\le\;
	\sqrt{\frac{C_\cov\,\beta}{c_0(t-1)}} .
\end{equation}
This step depends only on equation \eqref{eq:approx-stable-coverage}, the pooled-regression event $E_{\mathrm{reg}}$, and Bellman completeness. It does not use exact basis decomposition.

Finally, combining equations \eqref{eq:approx-cs}, \eqref{eq:approx-conflict-deployed}, and \eqref{eq:approx-in-sample}, for every $t\ge t_0$ and $h\in[H]$,
\begin{equation}\label{eq:approx-per-stage}
	\mathbb{E}_{d_h^{(t)}}\!\left[\delta_h^{(t)}\right] \;\le\; \bigl(\sqrt{1+\Gamma}+\varepsilon\bigr)\sqrt{\frac{C_\cov\,\beta}{c_0(t-1)}}.
\end{equation}
We split the regret in equation \eqref{eq:approx-regret-reduction} at the burn-in time $t_0$ and bound each part.

For each $t<t_0$, $\sum_h\mathbb{E}_{d_h^{(t)}}[\delta_h^{(t)}] = V_1^{f^{(t)}}(s_0;w_t)-\mathcal{J}(\pi^{(t)};w_t)\le H$, since both quantities lie in $[0,H]$. Hence, the burn-in contribution is at most $H\,t_0$.

Substituting equation \eqref{eq:approx-per-stage} and using
$\sum_{t=t_0}^T(t-1)^{-1/2}\le\sum_{u=1}^{T-1}u^{-1/2}\le 2\sqrt{T}$,
\begin{align*}
	\sum_{t=t_0}^T\sum_{h=1}^H \mathbb{E}_{d_h^{(t)}}\!\left[\delta_h^{(t)}\right] &\le H\bigl(\sqrt{1+\Gamma}+\varepsilon\bigr) \sqrt{\frac{C_\cov\,\beta}{c_0}} \sum_{t=t_0}^T\frac{1}{\sqrt{t-1}}
    &\le 2H\bigl(\sqrt{1+\Gamma}+\varepsilon\bigr) \sqrt{\frac{C_\cov\,\beta\,T}{c_0}} .
\end{align*}

Combined together, substituting $t_0=\mathcal{O}(C_\cov K\beta)$, and treating $t_0$ as a fixed constant yields
\[ \mathsf{Reg}_K(T) \le \mathcal{O}\!\left(H\sqrt{C_\cov \beta T}\cdot \bigl(\sqrt{1+\Gamma}+\varepsilon\bigr) + H C_\cov K \beta\right),\]
which leads to equation \eqref{eq:approx-regret}.

Lastly, we connect basis-relative regret $\mathsf{Reg}_K(T)$ to the cumulative regret $\mathsf{Reg}(T)$ against the true optimal policy via a clean two-term decomposition. 
For a fixed basis $\{\pi^\star_{\basis 1},\ldots,\pi^\star_{\basis K}\}$, define the cumulative basis-approximation error along the realized preference sequence as
\begin{equation*}%
	A_K(T) = \sum_{t=1}^T \!\left[\mathcal{J}(\pi^\star_{w_t};w_t) - \mathcal{J}_K^\star(w_t)\right].
\end{equation*}
Since every basis policy is feasible, $\mathcal{J}_K^\star(w_t)\le \mathcal{J}(\pi^\star_{w_t};w_t)$ for every $t$, so $A_K(T)\ge 0$. The quantity $A_K(T)$ depends on the basis $\{\pi^\star_{\basis k}\}$ and the preference distribution $\cP$, but \emph{not} on the algorithm.

First, we present the exact decomposition of regret relative to the true optimal policy.
By telescoping the comparator $\mathcal{J}_K^\star(w_t)$ in and out of the regret sum,
\begin{equation}\label{eq:true-regret-basis-decomp}
	\mathsf{Reg}(T) = \underbrace{\sum_{t=1}^T\!\left[\mathcal{J}(\pi^\star_{w_t}; w_t)-\mathcal{J}_K^\star(w_t)\right]}_{A_K(T)\ } +  \underbrace{\sum_{t=1}^T\!\left[\mathcal{J}_K^\star(w_t)-\mathcal{J}(\pi^{(t)};w_t)\right]}_{\mathsf{Reg}_K(T)}.
\end{equation}
Proposition~\ref{prop:approx-regret} controls the second term; the first depends only on the chosen basis.
Second, we bound $A_K(T)$ when optimal policies are also approximately in the hull.
Following Assumption~\ref{assm:approx_completeness}, we have $d_h^{\pi^\star_w}=q_{w,h}+e_{w,h}$. Since scalarized rewards $r_h^w(s,a)$ are bounded in $[0,1]$, its inner product with approximation error $e_{w,h}$ has
\[ \langle e_{w,h},r_h^w\rangle \le \bigl|\langle e_{w,h},r_h^w\rangle\bigr| \le \|e_{w,h}\|_{\bar d_h^{-1}} \sqrt{\mathbb{E}_{\bar d_h}\!\left[(r_h^w)^2\right]} \le \varepsilon .
\]
Using the same mixing vector $\nu_w$ at every stage,
\begin{align*}
	\mathcal{J}(\pi^\star_w; w) &= \sum_{h=1}^H\langle d_h^{\pi^\star_w},r_h^w\rangle \le \sum_{h=1}^H\langle q_{w,h},r_h^w\rangle + H\varepsilon\\
	&= \sum_{k=1}^K\nu_w(k)\mathcal{J}(\pi^\star_{\basis k}; w) + H\varepsilon\le
	\mathcal{J}_K^\star(w)+H\varepsilon.
\end{align*}
Summing over $t\in[T]$ yields
$A_K(T) \le H\varepsilon T.$
Lastly, combining equation \eqref{eq:true-regret-basis-decomp} with the above, and Proposition~\ref{prop:approx-regret}, we have 
\begin{equation*}%
	\mathsf{Reg}(T) \le\mathcal{O}\!\left(H\varepsilon T+H\sqrt{C_\cov\beta T}\cdot\bigl(\sqrt{1+\Gamma}+\varepsilon\bigr) + HC_\cov K\beta\right).
\end{equation*}
\end{proof}

\section{Omitted Experiments and Implementation Details}\label{app:experiments}

In this section, we describe several omitted experimental results as well as the implementation details that are left out from Section \ref{sec:experiment}.

We conduct the MO-Gymnasium experiments on an AMD Ryzen 7 9700X CPU, with a total training time of approximately 72 hours. Meta-World and RLHF experiments are performed on a server with an Intel Xeon E5-2623 v3 CPU and NVIDIA RTX A6000 GPUs. The Meta-World training takes 12 hours, and the RLHF experiment takes 1 hour.

Next, we define the performance metrics used in our experiments.
We compute $\Gamma$ based on the $\chi^2$-divergence between the state occupancy distributions of single-objective basis policies. To account for the temporal diffusion effect described above, we use the time-averaged state occupancy 
\begin{align*} 
    \bar{d}^\pi(s) = \frac{1}{H}\sum_{t=0}^H P(s_t=s), \text{ then }
    \Gamma = \max_{k} \sum_{s \in \mathcal{S}} \frac{(\bar{d}^{\pi_k}(s) - \bar{d}^{\avg}(s))^2}{\bar{d}^{\avg}(s)},
\end{align*}
where $\pi_k$ is the policy trained on objective $k$ alone, and $\bar{d}^{\avg}$ is the uniform mixture of all basis occupancies.

We train a kernel ridge regression model $\hat{f}$ on the dataset of $\{(\mathbf{m}_i, y_i)\}$. The attribution of a conflicting objective $j$ (where $j \in \{5, \dots, 9\}$) is calculated as the predicted marginal gain $\hat{f}(\mathbf{m}_{\target} \cup \{j\}) - \hat{f}(\mathbf{m}_{\target}),$
where $\mathbf{m}_{\target}$ is the mask containing only the target objective.

Lastly, we provide the detailed hyperparameter settings used in our experiments to ensure reproducibility. Table~\ref{tab:routed_ensemble_params} lists the parameters for the routed ensemble algorithm across different benchmarks.
For routed ensemble, we adapt the base policy-learning algorithm for each environment. We use Soft Modularization~\citep{yang2020multi} for Meta-World, Q-Pensieve~\citep{hung2022q} for continuous MO-Gymnasium tasks, and PQL~\cite{van2014multi} for discrete MO-Gymnasium tasks.
For all baselines, we use the hyperparameters from their official implementations/repositories to ensure a fair comparison.
Each baseline is trained for $5 \times 10^5$ time steps on continuous benchmarks and $10^4$ time steps on discrete benchmarks.

\begin{table}[t]
    \centering
    \caption{Hyper-parameters for the routed ensemble procedure used in Algorithm~\ref{alg_routing}.}
    \label{tab:routed_ensemble_params}\vspace{0.075in}
    {\small\begin{tabular}{lccc}
        \toprule
        \textbf{Parameter} & \textbf{Meta-World} & \textbf{MO-Gymnasium} \\
        \midrule
        Number of basis policies & $3$ & $3$ \\
        Base optimizer & Adam & Adam \\
        SAC buffer size & $10^6$ & $10^6$ \\
        Learning rate (Actor) & $3\times10^{-4}$ & $3\times10^{-4}$ \\
        Learning rate (Critic) & $3\times10^{-4}$ & $3\times10^{-4}$ \\
        Discount factor & $0.99$ & $0.99$ \\
        Batch size & $1280$ & $256$ \\
        Network architecture & MLP $256\times256$ & MLP $256\times256$ \\
        \bottomrule
    \end{tabular}}
\end{table}

We implement kernel-based objective attribution in Algorithm~\ref{alg_krr} using a RBF kernel defined as $\psi(S_i, S_j) = \exp(-\gamma \|S_i - S_j\|^2)$. We determine the kernel bandwidth $\gamma = 0.1$ and the regularization coefficient $\lambda = 0.1$ via cross-validation on the validation set.

\subsection{Synthetic environments}\label{app_syn_env}

\textbf{Multi-objective environments.} We utilize a 2D parametric optimization environment. The policy is parameterized by a coordinate vector $\theta \in \mathbb{R}^2$. The agent aims to maximize a scalarized return $J(\theta, w) = w \cdot r_1(\theta) + (1-w) \cdot r_2(\theta)$, where $w \in [0, 1]$ represents the preference weight. We optimize $\theta$ using the Adam optimizer across a sweep of discrete weights $w$. All environments include a small regularization penalty $-0.1\|\theta\|_2^2$. We define three distinct regimes:

\emph{Case 1 (Aligned objectives).} The reward functions are quadratic penalties centered at proximal locations $\mu_1 = (1, 1)$ and $\mu_2 = (1.2, 1.2)$. This configuration simulates tasks where improving one objective implicitly benefits the other.

\emph{Case 2 (Conflicting objectives with a linear frontier).} We construct a non-convex reward landscape to induce a disjoint set of optimal policies. The targets are separated at $\mu_1 = (-2, 0)$ and $\mu_2 = (2, 0)$. Crucially, we introduce a Gaussian barrier penalty $P(\theta) = -10 \exp(-\theta_x^2)$ centered at the origin. This barrier prevents the optimization from settling in the intermediate region. Consequently, the optimal deterministic policies collapse to the extreme corners, and the Pareto frontier is formed by the linear convex hull of these solutions.

\emph{Case 3 (Conflicting objectives with a concave frontier).} We remove the central barrier while maintaining the separated targets $\mu_1 = (-2, 0)$ and $\mu_2 = (2, 0)$. The reward functions are standard negative Euclidean distances. This convex landscape allows for smooth interpolation between objectives. It results in a strictly concave Pareto frontier, in which intermediate policies outperform the linear combination of the corner solutions.

\paragraph{Nonlinear synergy.} We introduce an environment for conducting synthetic experiments to demonstrate synergy effects. This environment creates a synergy trap in which two objectives are individually detrimental or neutral but highly beneficial when combined.

We define a 2D parametric optimization task with $M=12$ distinct objectives, indexed by $j \in \{0, 1, \dots, M-1\}$.
The objectives are arranged equidistantly on a unit circle (a clock face). The goal vector for the $j$-th objective, $v_j \in \mathbb{R}^2$, is defined by the angle $\phi_j = \frac{\pi}{2} - j \cdot \frac{2\pi}{M}$:$$v_j = [\cos(\phi_j), \sin(\phi_j)]^\top.$$
Here, objective $0$ (north) is the designated target with $v_{\target} = [0, 1]^\top$.
The environment employs a specialized reward function $R(a, S)$ for an action $a \in \mathbb{R}^2$ given a subset of active objectives $S \subseteq \{0, \dots, M-1\}$, represented by a binary mask $m \in \{0,1\}^M$.
The reward consists of a standard linear component and a nonlinear synergy component:
\[R(a, S) = {a^\top \left(\sum_{j \in S} v_j\right)} + {\mathbb{I}_{3 \in S \land 9 \in S} \cdot \kappa_{\mathrm{syn}} \cdot (a^\top v_{\target})}, \]
where $\kappa_{syn}=5.0$ is the synergy bonus coefficient. 

The agent is modeled as a stochastic policy $\pi_\theta(a) = \mathcal{N}(\mu_\theta, \Sigma)$, where $\mu_\theta = \tanh(\theta)$ and $\Sigma = \sigma^2 I$ (with $\sigma=0.2$). The parameter $\theta \in \mathbb{R}^2$ is optimized via policy gradient, which optimizes the parameters of a stochastic policy by maximizing expected reward through Monte Carlo sampling, using the Adam optimizer with a learning rate of $0.1$ for a horizon of $T_{ep}=60$ steps.

\paragraph{A tabular environment on a grid.}
To validate the theoretical interactions between the objective conflict coefficient $\Gamma$ and the coverability coefficient $C_{cov}$, we implement a controlled tabular environment. This setup allows for the exact calculation of occupancy measures and $\chi^2$-divergences, avoiding the approximation errors inherent in deep RL settings.

We consider a MOMDP modeled as an $N \times N$ grid (default $N=5$), defined by the tuple $\mathcal{M}=(\mathcal{S},\mathcal{A},Pr,r,\gamma)$.
Consistent with the conflicting-objectives regime, we construct $m$ objectives based on spatial navigation to distinct goal coordinates. Let $\mathcal{G} = \{g_1, \dots, g_m\} \subset \mathcal{S}$ be a set of diverse goal states. The vector-valued reward function $r(s,a) \in \mathbb{R}^m$ is defined using the negative Manhattan distance ($\ell_1$ norm): $r_k(s,a) = -\| \text{coord}(s') - \text{coord}(g_k) \|_1,$ where $s'$ is the next state and $\text{coord}(\cdot)$ maps a state index to its $(x,y)$ grid position.

We compute optimal policies exactly via soft value iteration. For a given preference weight $w \in \Delta^{m-1}$, the scalarized reward is $r_w = \langle w, r \rangle$. The policy $\pi_w^*$ is derived from the soft Bellman operator:
\[ V(s) = \tau \log \sum_{a} \exp\left(\frac{Q(s,a)}{\tau}\right), \quad \pi(a|s) \propto \exp\left(\frac{Q(s,a)}{\tau}\right), \]
with temperature $\tau=0.5$. The exact state-action occupancy distribution $d^\pi \in \mathbb{R}^{|\mathcal{S}||\mathcal{A}|}$ is computed by solving the linear system:
\[ d^\pi = (I - \gamma P_\pi)^{-1} d_0, \]
where $P_\pi$ is the policy-induced transition matrix and $d_0$ is the uniform initial distribution.

We explicitly calculate the conflict coefficient. We define the target mixture distribution $\bar{d}$ as the uniform average of the occupancies of the single-objective experts:
\[ \bar{d}(s,a) = \frac{1}{m} \sum_{k=1}^{m} d^{\pi_k^\star}(s,a). \]
The conflict coefficient is then measured as the $\chi^2$-divergence between a candidate policy's occupancy and this target mixture:
\[ \Gamma(\pi) = \chi^2(d^\pi \| \bar{d}) = \sum_{s,a} \frac{(d^\pi(s,a) - \bar{d}(s,a))^2}{\bar{d}(s,a) + \epsilon}. \]
We utilize this metric to compare coverability in Section \ref{sec_abla}. We find that a single scalarized policy $\pi_{single}$ exhibits a high $\chi^2$-divergence, indicating poor coverage of the Pareto frontier. In contrast, a routed ensemble with $K$ basis policies significantly reduces this divergence.

\subsection{Additional empirical results}\label{app_omitted}

\textbf{Empirical evaluation of the basis assumption.}
We measure the residual $\epsilon$ obtained by approximating preference-policy occupancies with convex combinations of basis-policy occupancies. Tabular occupancies are computed exactly, while benchmark occupancies are estimated from visitation distributions. Table~\ref{tab:occupancy_hull_residuals} reports the mean and maximum residuals over the evaluated preferences.

\begin{table*}[h]
\centering
\caption{Convex-hull occupancy residuals for exact tabular occupancies and empirical benchmark visitation distributions. Mean and maximum residuals are computed over the evaluated preferences.}
\label{tab:occupancy_hull_residuals}
\small
\setlength{\tabcolsep}{5pt}
\begin{tabular}{llccc}
\toprule
Setting & Basis construction & Number of policies & Mean $\epsilon$ & Max $\epsilon$ \\
\midrule
Tabular, aligned & Exact & 1 & $\sim 10^{-18}$ & $\sim 10^{-17}$ \\
Tabular, linear & Exact & 2 & $\sim 10^{-9}$ & $\sim 10^{-8}$ \\
Tabular, concave & Exact & 3 & 0.186 & 0.55 \\
Tabular, concave & Exact & 21 & $1.45\times 10^{-4}$ & $6.10\times 10^{-4}$ \\
MO-Hopper & Candidate-selected & 3 & 0.443 & 0.80 \\
MO-Hopper & Candidate-selected & 15 & 0.231 & 0.51 \\
MO-Hopper & Learned by Alg.~\ref{alg_routing} & 3 & 1.33 & 1.42 \\
Deep-Sea-Treasure & Learned by Alg.~\ref{alg_routing} & 6 & 0.15 & 0.75 \\
Fruit-Tree & Oracle & 6 & 0.553 & 0.85 \\
GridWorld, 21 distinct optima & Restricted learner & 2 & 0.475 & 3.474 \\
\bottomrule
\end{tabular}
\end{table*}

In the concave tabular setting, the mean residual decreases from $0.186$ at three bases to $0.088$, $0.024$, and $1.45\times 10^{-4}$ at five, nine, and 21 bases, respectively. On MO-Hopper, increasing the candidate-selected basis from three to 15 policies decreases the mean residual from $0.443$ to $0.231$, whereas the three policies learned by Algorithm~\ref{alg_routing} have mean residual $1.33$. These measurements show that a larger candidate basis can approximate the evaluated occupancies more closely, but the learned ensemble need not recover the exact structure assumed in the analysis.

\paragraph{Severe basis under-specification.}
In a controlled GridWorld with 21 distinct preference-optimal policies, we restrict the learner to two policies for 3000 episodes. The resulting mean and maximum occupancy residuals are $0.475$ and $3.474$, and the cumulative value excess is $2780.4$. This test measures performance when the available ensemble is substantially smaller than the set of preference-optimal policies.

\paragraph{Comparison on RLHF tasks.}
We evaluate the extension of our approach to RLHF settings on synthetic data, HelpSteer2~\cite{wang2024helpsteer}, and UltraFeedback~\cite{cui2023ultrafeedback}. For the synthetic experiment, we construct a synthetic preference dataset that captures the helpfulness-safety trade-off in RLHF. Each prompt is paired with candidate textual responses corresponding to different response styles, while rewards are generated from a controlled two-objective preference function. This allows us to study expert routing under interpretable language scenarios while retaining access to an oracle policy. For the real-data experiments, we evaluate Qwen3-1.7B and Llama-3.2-1B on both HelpSteer2 and UltraFeedback, yielding four model--dataset configurations. All four configurations follow the same training protocol, with per-basis LoRA adapters trained using PPO.

\begin{table}[h]
    \centering
    \caption{Summary of trajectory attribution results on two MO-Gymnasium benchmarks, Deep Sea Treasure and Fruit Tree.}\vspace{0.075in}
    {\small\begin{tabular}{l|cc}
    \toprule
                       & Deep Sea Treasure & Fruit Tree \\ \hline
    Influence functions      & $0.28_{\pm0.51}$ & $0.30_{\pm0.57}$ \\
    Linear surrogate models                     & $0.62_{\pm0.39}$ & $0.71_{\pm0.53}$ \\
    TracIn                   & $0.70_{\pm0.40}$ & $0.75_{\pm0.45}$ \\
    \textbf{Algorithm \ref{alg_krr}} & $\mathbf{0.75}_{\pm0.14}$ & $\mathbf{0.81}_{\pm0.18}$ \\
    \bottomrule
    \end{tabular}}
\end{table}

\paragraph{Choice of $N$.} Finally, we present comparison results by varying the size of the ensemble policy. In Figure \ref{fig_ablation_N}, we illustrate the choice of $N$ on the three datasets, including MO-HalfCheetah2d, MO-Hopper, and MO-Ant environments.

\begin{figure}[h]
    \centering
    \begin{subfigure}{0.315\linewidth}
        \centering
        \includegraphics[width=0.75\linewidth]{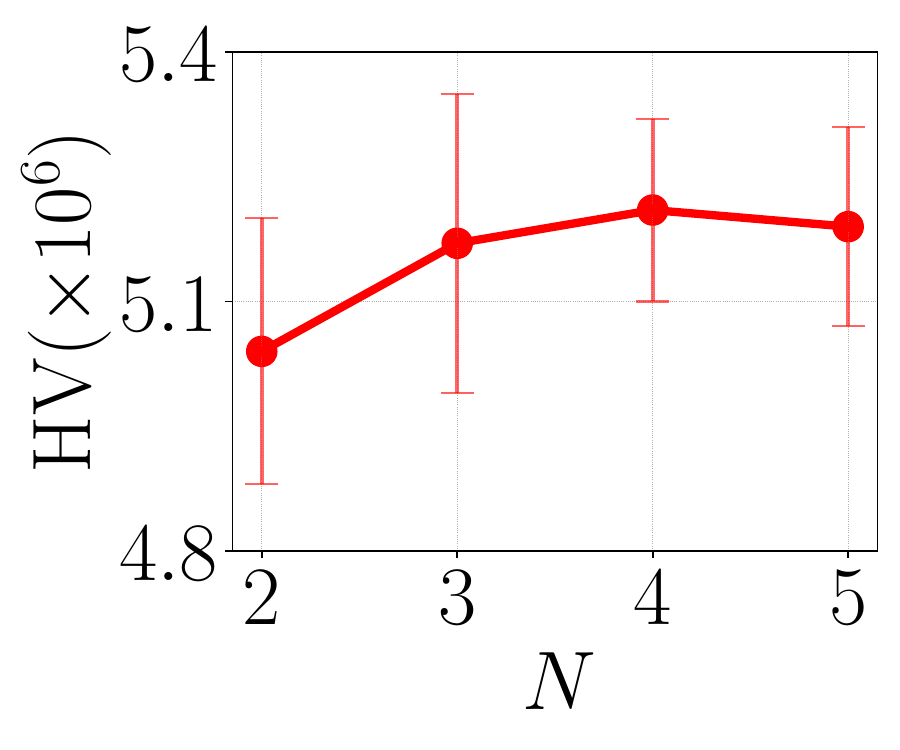}
    \end{subfigure}\hfill
    \begin{subfigure}{0.315\linewidth}
        \centering
        \includegraphics[width=0.75\linewidth]{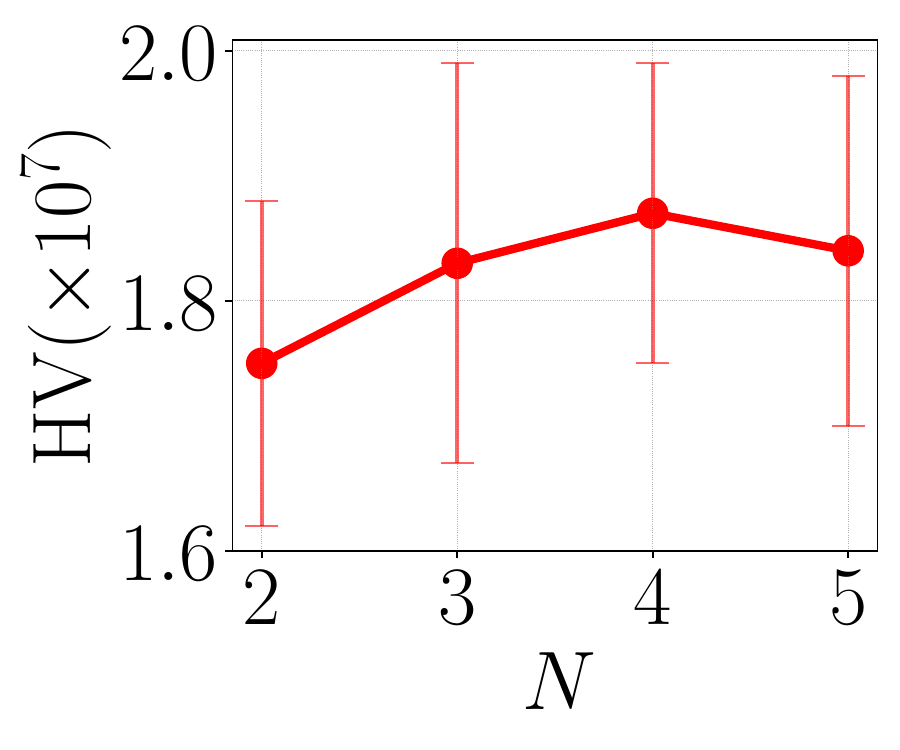}
    \end{subfigure}\hfill
    \begin{subfigure}{0.315\linewidth}
        \centering
        \includegraphics[width=0.75\linewidth]{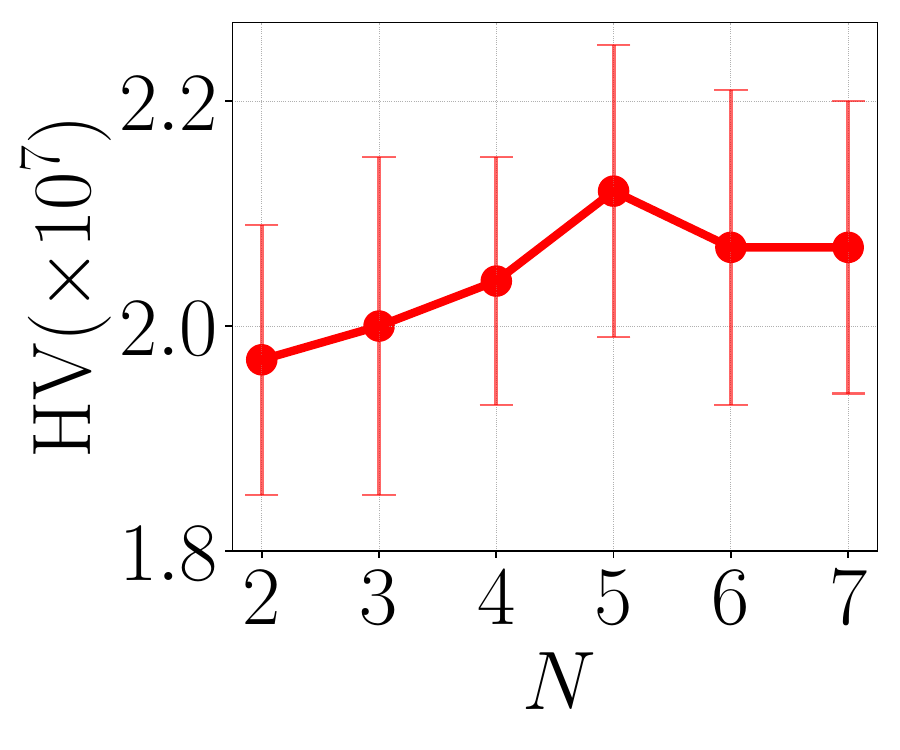}
    \end{subfigure}
    \vspace{-0.00in}
    \caption{Ablation analysis of the choice of the ensemble size $N$, measured by hypervolume (HV).}\label{fig_ablation_N}
\end{figure}

\subsection{Proximal policy optimization}\label{app_ppo}

In proximal policy optimization (PPO), access to the exact $A^\pi$ is not available. Instead, one must use an advantage estimator $\widehat{A}_t$ computed from rollout trajectories and a learned value function $V_\phi$. A common choice is generalized advantage estimation (GAE), where the temporal-difference residual at time $t$ is
\[ \delta_t = r_t + \gamma V_\phi(s_{t+1}) - V_\phi(s_t), \]
and the advantage estimate is
\[ \widehat{A}_t = \sum_{l=0}^{L-1} (\gamma \lambda)^l \delta_{t+l}, \]
for some truncation horizon $L$, the discount factor $\gamma \in [0, 1)$, and GAE parameter $\lambda \in [0,1]$. More generally, we use $\widehat{A}_t$ to denote any scalar advantage estimate derived from the rollout and a baseline; our theory treats these $\widehat{A}_t$ as given and is independent of the particular estimator.

Training proceeds in rounds $k = 0,1,2,\ldots$. In round $k$, the current policy $\pi_{\theta^{(k)}}$ interacts with the environment and collects a rollout buffer $B^{(k)} = \{ z_i^{(k)} \}_{i=1}^{n_k},$
where each record
\[ z_i^{(k)} = (s_i, a_i, r_i, \log \pi_{\theta^{(k)}}(a_i \mid s_i), v_i, \widehat{A}_i) \]
contains the transition $(s_i, a_i, r_i)$ together with auxiliary quantities such as the log-probability under the current policy, a discounted value estimate $v_i$, and an advantage estimate $\widehat{A}_i$. These records are the atomic units used in PPO updates and form the basic entities for local attribution. 

Given a buffer $B^{(k)}$, PPO performs a number of gradient steps to update $\theta$ from $\theta^{(k)}$ to $\theta^{(k+1)}$ by maximizing a clipped surrogate objective.
For a single record $z = (s,a,r,\ldots)$, let $\ell_{\mathrm{PPO}}(\theta; z)$ denote the per-record loss corresponding to the standard clipped PPO objective (possibly combined with value and entropy terms). The empirical PPO loss over the buffer is then
\[ \hat{L}^{(k)}(\theta)
= \frac{1}{n_k} \sum_{i=1}^{n_k} \ell_{\mathrm{PPO}}(\theta; z_i^{(k)}), \]
and PPO updates $\theta$ to approximately minimize $\hat{L}^{(k)}(\theta)$, starting from $\theta^{(k)}$.

\end{document}